\documentclass[10pt,twocolumn,letterpaper]{article}

\usepackage[pagenumbers]{cvpr} % To force page numbers, e.g. for an arXiv version

\usepackage{graphicx}
\usepackage{url}            % simple URL typesetting
\usepackage{booktabs}       % professional-quality tables
\usepackage{nicefrac}       % compact symbols for 1/2, etc.
\usepackage{bm}
\usepackage{algorithmic}
\usepackage{algorithm}
\usepackage{multirow}
\usepackage[title]{appendix}

\usepackage[spacing]{microtype}
\usepackage{amsmath, amsthm, amsfonts}

\newcommand{\Vc}[1]{\bm{#1}}         % 向量
\newcommand{\Mat}[1]{\bm{#1}}           % 矩阵
\newcommand{\T}{{\top}}                 % 转置
\newcommand{\cT}{{*}}                   % 共轭转置
\newcommand{\pinv}{{\dagger}}           % Moore-Penrose伪逆
\newcommand{\dataset}[1]{\mathcal{#1}}  % 数据集
\newcommand{\SFN}[1]{{\lVert{#1}\rVert_\mathrm{F}^2}}   % Frobenius范数的平方
\newcommand{\onehot}[1]{{\operatorname{onehot}(#1)}}    % One-hot函数
\newcommand{\argmin}[1]{\underset{#1}{\operatorname{argmin\ }}} % argmin函数
\newtheorem{lemma}{Lemma}

\newtheorem{theorem}{Theorem}

\definecolor{cvprblue}{rgb}{0.21,0.49,0.74}
\usepackage[pagebackref,breaklinks,colorlinks,allcolors=cvprblue]{hyperref}

\def\paperID{11294} % *** Enter the Paper ID here
\def\confName{CVPR}
\def\confYear{2025}

\title{AFL: A Single-Round Analytic Approach for Federated Learning with Pre-trained Models}

\author{\textbf{Run He}$^{1}$, \textbf{Kai Tong}$^{1}$, \textbf{Di Fang}$^{1}$, \textbf{Han Sun}$^{2,3}$, \textbf{Ziqian Zeng}$^{1}$,  \textbf{Haoran Li}$^{4}$, \\ \textbf{Tianyi Chen}$^{5}$, \textbf{Huiping Zhuang}$^{1*}$ \\
$^{1}$ South China University of Technology \quad
$^{2}$ Tsinghua University\\
$^{3}$ Beijing National Research Center for Information Science and Technology \\
$^{4}$ The Hong Kong University of Science and Technology \quad
$^{5}$ Microsoft\\
{$^{*}$\tt\small corresponding: hpzhuang@scut.edu.cn}
}

\begin{document}
\maketitle
\begin{abstract}
    In this paper, we introduce analytic federated learning (AFL), a new training paradigm that brings analytical (i.e., closed-form) solutions to the federated learning (FL) with pre-trained models. Our AFL draws inspiration from analytic learning---a gradient-free technique that trains neural networks with analytical solutions in one epoch. In the local client training stage, the AFL facilitates a one-epoch training, eliminating the necessity for multi-epoch updates. In the aggregation stage, we derive an absolute aggregation (AA) law. This AA law allows a single-round aggregation, reducing heavy communication overhead and achieving fast convergence by removing the need for multiple aggregation rounds. More importantly, the AFL exhibits a property that \textit{invariance to data partitioning}, meaning that regardless of how the full dataset is distributed among clients, the aggregated result remains identical. This could spawn various potentials, such as data heterogeneity invariance and client-number invariance. We conduct experiments across various FL settings including extremely non-IID ones, and scenarios with a large number of clients (e.g., $\ge 1000$). In all these settings, our AFL constantly performs competitively while existing FL techniques encounter various obstacles. Our codes are available at \url{https://github.com/ZHUANGHP/Analytic-federated-learning}.
\end{abstract}    
\section{Introduction}
Federated learning (FL) \cite{FedAvg2017ICML} aims to collectively train a machine learning model over data silos by aggregating their individual trained weights, while preserving the privacy of their source data. This training paradigm has received high popularity, particularly in sensitive domains where data privacy is crucial, such as in banks \cite{FedLoan2020ICDMW, FFD2019} and hospitals \cite{FedMed12020NMI, FedMed22020IEEESensor}.

Conventional FL techniques rely on weight aggregation among clients over multiple rounds of training. The objective is to achieve convergence and approximate its joint-training counterpart, where all clients' data are accessible in a single location. To accomplish this, many contributions have been made. One widely recognized method is FedAvg \cite{FedAvg2017ICML}. Relying on a large number of aggregation rounds, the FedAvg employs a simple yet effective weight averaging technique across local clients. Building upon this, various methods have been proposed (e.g., the FedProx \cite{FedProx2020MLSYS} and the FedNova \cite{fednova2020nips}), each with its own specific focus within the field of FL.

However, training a model from scratch via FL can be computationally intensive and demanding in terms of communication bandwidth, especially with large models and numerous participating clients. Several efforts have explored utilizing pre-trained models to mitigate these challenges \cite{FedPEFT2023CoLLAs,FedPETuning2023ACLfindings}. Typically, this involves freezing the backbone and only updating and sharing lightweight parameters, such as prototypes \cite{FedProto2022AAAI, FedPCL2022NeurIPS} or prompts \cite{FedPR2023CVPR, pFedPG2023ICCV}, to reduce the substantial training costs.
 
Although leveraging pre-trained models can circumvent the high costs associated with training backbones from scratch, existing FL techniques with pre-trained models are primarily based on a gradient-based iterative approach, necessitating iterative optimization on each client and multi-round aggregation across clients. The gradient-based optimization used in the existing FL faces various challenges and imposes several constraints. The faced challenges include, but are not limited to: 1) \textit{Data heterogeneity}, where the data distribution in each client is not independently identical (non-IID), even with mutually exclusive data categories across different clients (i.e., pathological distribution), 2) \textit{Large client number}, where the aggregation involving a significant number of clients (i.e., $\ge1000$) can lead to substantial performance degradation in FL systems as the client count increases \cite{FedNIID_survey2021ICDE}, 3) \textit{Slow convergence:} where FL methods may struggle to converge within limited communication rounds, especially in severe non-IID scenarios, and 4) \textit{High communication cost}, where multi-round aggregation in existing FL methods escalates the communication costs associated with parameter sharing between clients and servers. 

% Moreover, the gradient-based approach imposes several limitations.  For instance, an \textit{Intensive aggregation} process is necessary, prolonging FL training with numerous aggregation rounds. Additionally, FL techniques often struggle with a \textit{Lengthy hyperparameter tuning} process, further intensifying the already demanding multi-round aggregation.

In this paper, we propose a new FL training framework named analytic federated learning (AFL), which provides a single-round aggregation for federated learning with pre-trained models. The AFL draws inspiration from analytic learning \cite{karnet2018, pil2001, brmp2021}---a gradient-free technique with a closed-form solution obtained from reshaping the network training into linearized formulation. The AL paradigm receives several benefits over gradient-based techniques. First, it is gradient-free, thereby avoiding gradient-related issues, such as vanishing and exploding gradients. Second, the analytical solution frees AL from convergence issues during training. Also, the AL requires only one visit to the dataset while gradient-based mechanism usually needs hundreds of epochs or beyond. These properties are attractive in FL to accomplish fast convergence and low communication cost.
% Relevant incorporations of AL have been demonstrated successfully in several domains such as continual learning (e.g., analytic continual learning \cite{ACIL2022NeurIPS,GKEAL2023CVPR2023,DSAL2024AAAI}).
Here, we are able to incorporate this mechanism into the FL domain to overcome limitations inherent in gradient-based techniques. Our contributions are summarized as follows:

%#1 no paramter tuning (an important one)
%#2 one-epoch
%#3 one-aggregation
%#4 client vairant
%#5 convergence-free
%#6

% 使用enumerate更规范，且CVPR模板的enumerate比bullet更节约空间。
\begin{itemize}
    \item We propose the AFL, a gradient-free FL framework with analytical (closed-form) solutions. These analytical solutions apply both in the local client training stage and the aggregation stage.
    \item In the local stage, we adopt a pre-trained network to harness input embeddings, and formulate the training in each client into a localized linear regression problem. This leads to a least squares (LS) based \textit{one-epoch client training}, eliminating the need for multi-epoch training and enabling fast convergence in local clients.
    \item In the aggregation stage, we derive an absolute aggregation (AA) law in analytical form, optimally establishing a \textit{single-round aggregation}. That is, the aggregation happens only once, avoiding multiple FL rounds that bring high communication costs. Additionally, in scenarios where the AA law becomes suboptimal due to a large number of clients, we introduce a regularization intermediary (RI) process to restore its optimality.
    \item Owing to analytical solutions, the AFL exhibits a property that \textit{invariance to data partitioning}. This means that regardless of how the full dataset is distributed (e.g., non-IID) among local clients, the result remains identical. This property spawns several appealing characteristics: i) \textit{Data heterogeneity invariance} where the result is invariant to arbitrary heterogeneous data partition scenarios. ii) \textit{Client-number invariance}, which produces identical results regardless of the number of clients involved. 
    \item We conduct extensive experiments spanning diverse scenarios, including a wide variety of non-IID partitions and large client number (up to 1000) settings. Our AFL consistently showcases competitive performance throughout all these settings when compared with other methods.
\end{itemize}

\section{Related Works}
In this section, we review existing related FL literature. Additionally, we explore various AL techniques and their variants to reveal their underlying mechanisms.

%In this section, We review existing relevant FL literature and compare them with our approach, demonstrating the advantages of our method in reducing communication overhead and enhancing convergence speed and model accuracy in non-IID environments. Additionally, we explore various AL technologies and their variants to elucidate their potential mechanisms.

\subsection{Federated Learning Methods}

Following the FedAvg \cite{FedAvg2017ICML}, to address non-IID issues in FL, various methods have been proposed. One common approach involves assessing the significance of parameters during aggregation to ensure that local updates do not diverge substantially from the global model. For instance, the FedProx \cite{FedProx2020MLSYS} restricts the size of local updates, while the FedNova \cite{fednova2020nips} employs a normalized averaging method to eliminate target inconsistency while maintaining fast error convergence. These methods are frequently used as baselines, and we compare our results against them in our experiments. Another set of methods focuses on determining adaptive aggregation weights obtained from multiple clients. The FedLAW \cite{FedLAW2023ICML} learns these weights to achieve a global model with state-of-the-art performance, though it requires a proxy dataset to learn the weights, making the results sensitive to the selection of the proxy dataset. To address this sensitivity, the FedCDA \cite{FedCDA2024ICLR} proposes a proxy-free method that reduces each client's deviation from the local models of other participants and selects a local model from its multiple recent models acquired over several rounds.

Some methods address the parameter order mismatch issue across clients, which can occur during global aggregation. The Fed2 \cite{yu2021fed2} designs a model structure adaptation method to ensure explicit feature allocation across different network structures. Similarly, the method in \cite{fedneuron_2022_CVPR} seeks a position-aware neuron to fuse position-related values (i.e., position encodings) into neuron outputs. Distillation methods \cite{Guo_2020_CVPR,NEURIPS2020_distributedistillation,Wu_Gong_2021aaai,Wang_2023_CVPR} represent another branch, where the average of logits from client models is used for the local model aggregation, thereby enhancing generalization. \cite{NEURIPS2020_lin} pioneers to apply knowledge distillation on the server side, transferring knowledge from multiple local models to the global model using an unlabeled proxy dataset. To overcome the limitation of using a proxy dataset, recent studies such as \cite{pmlr-zhu21b} and \cite{Zhang_2022_CVPR} suggest substituting the proxy dataset with generated data. 

Existing FL techniques have several significant drawbacks, including challenges with data heterogeneity, large client numbers, convergence issues and high communication costs. Our AFL framework addresses these issues by utilizing a gradient-free, closed-form analytic learning approach, avoiding gradient-related problems (e.g., multi-epoch training, convergence issues and multi-round aggregation).

\subsection{Analytic Learning}
The AL has been developed as a strategy to address issues associated with gradient-based update, such as gradient vanishing/exploding, divergence during iteration, and long training time due to multi-epoch training. The AL is also referred to as pseudoinverse learning \cite{pil2001} owing to its utilization of matrix inversion. The AL starts from shallow learning, which is investigated prior to the advent of deep networks in the realm of research. For instance, the radial basis network \cite{rbf1991univ} trains parameters using an LS estimation after performing a kernel transformation in the first layer. The multilayer AL \cite{2018icisToh,analytic-learning-finallayer2019noniterative} comes up with a one-epoch training style, using LS techniques to resolve linear segments transformed by nonlinear network. One instance of this method is the dense pseudoinverse autoencoder \cite{DensePILAE2021journal}, which uses LS solutions to combine shallow and deep features to train a stacked autoencoder layer-by-layer.

Nonetheless, earlier AL techniques train their weights by processing the entire dataset simultaneously, therefore facing memory challenge. This memory concern is alleviated by the block-wise recursive Moore-Penrose inverse \cite{brmp2021}, which equivalently replaces the joint learning with a recursive approach. This recursive equivalent property echoes well with the continual learning community. Naturally, analytic continual learning techniques \cite{ACIL2022NeurIPS,GKEAL2023CVPR2023,DSAL2024AAAI} adopt this equivalent characteristic, thrive in handling the catastrophic forgetting problem, and are invariant to the sequential data partition in continual learning. Our AFL draws inspiration from these adaptations, aiming to introduce similar equivalent patterns (e.g., invariant to heterogeneous data) to the FL community.
\begin{figure*}[t]
        \centering
        \includegraphics[width=\linewidth]{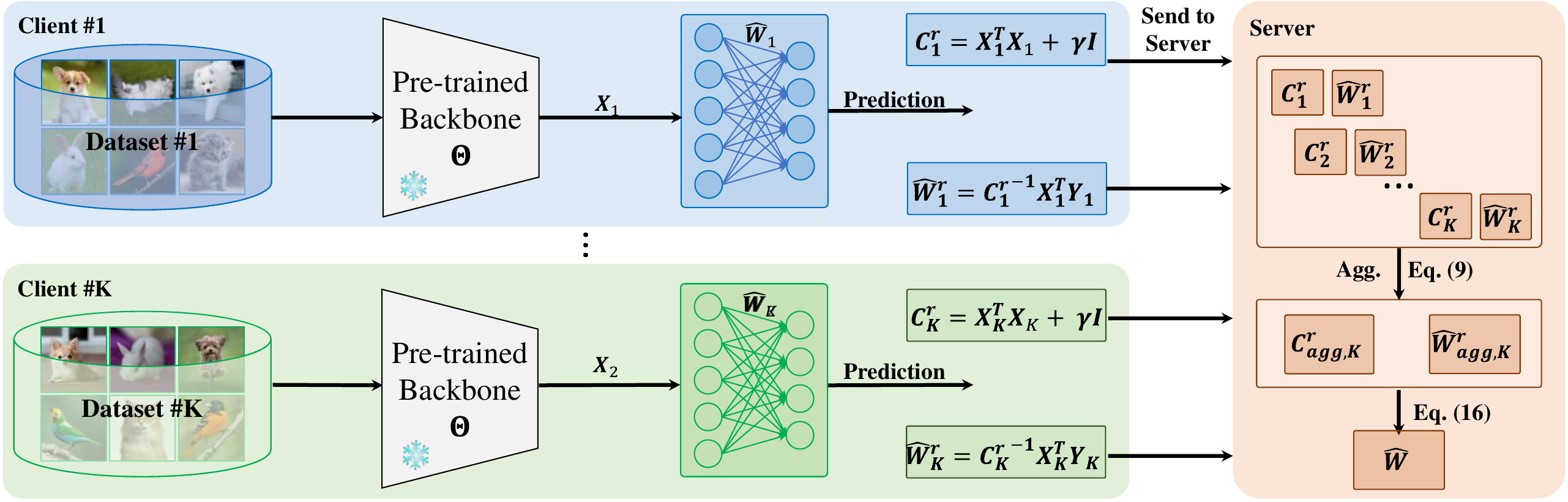}
        \caption{An overview of the AFL. During the local stage, each client calculates $\Mat{C}_{k}^{\text{r}}$ and $\Mat{\hat W}_{k}^{\text{r}}$ based on the same pre-trained backbone and its own dataset. The server obtained the $\Mat{C}_{\text{agg},K}^{\text{r}}$ and $\Mat{\hat W}_{\text{agg},K}^{\text{r}}$ then get $\Mat{\hat W}$ in the aggregation stage.}\label{fig:mainflow}
        \vskip -0.1in
    \end{figure*}

\section{Analytic Federated Learning}
In this section, we provide a detailed exposition of AFL derivations, organized into a local training stage and a centralized aggregation stage. In the local stage, a pre-trained backbone serves as a feature extractor, facilitating an AL network learning that allows the training to be completed in one epoch. In the aggregation stage, we introduce the AA law, establishing a single-round aggregation. We elaborate on AFL's invariance to data partitioning here, bringing benefits such as data heterogeneity invariance, client-number invariance and fast convergence in a single round. An Overview of the proposed AFL paradigm is depicted in Figure \ref{fig:mainflow}.

Prior to further developments, here let $\dataset{D} = \{\dataset{D}_{k}\}_{k=1}^{K}$ be the complete training data, where $\dataset{D}_{k} \sim \{\mathcal{X}_{k,i},y_{k,i}\}_{i=1}^{N_{k}}$ suggests an $N_{k}$-sample  sub-dataset accessible to the $k$-th client with $\mathcal{X}_{k,i}$ and $y_{k,i}$ representing the $i$-th input-label pair. In this paper, all these $K$ clients share the same backbone network $f_{\text{backbone}}$ parameterized by $\bm{\Theta}$ to map their inputs (e.g., $\mathcal{X}$) to embedding vectors.

\subsection{Local Stage: Localized Analytic Learning}
In this stage, each local client's network is trained using the AL technique. This involves transforming the neural network's classification head into a linear regression problem, thereby enabling the derivation of a closed-form LS solution.

At the initial step, client $k$ extracts its embedding vector $\Vc{x}_{k,i}$ by passing the $i$-th data $\mathcal{X}_{k,i}$ from $\dataset{D}_{k}$ through the frozen backbone network $f_{\text{backbone}}$, i.e., 
\begin{equation}
 \Vc{x}_{k,j} = f_{\text{backbone}}(\mathcal{X}_{k,j}, \bm{\Theta}) 
\end{equation}
where $ \Vc{x}_{k,j} \in\mathbb{R}^{1\times y_{\text{e}}}$, with $y_{\text{e}}$ indicating the embedding length. 

\par For the $k$-th client (with $N_{k}$ samples in $\dataset{D}_{k}$), we can stack the extracted embeddings and their corresponding one-hoted labels via mapping $\dataset{D}_{k} \sim \{\mathcal{X}_{k,i},y_{k,i}\}_{i=1}^{N_{k}}$ to $\dataset{\bar D}_{k} \sim \{\Mat{X}_{k}, \Mat{Y}_{k}\}$, i.e., 
\begin{equation}\label{eq_emb_label}
	\resizebox{0.43\textwidth}{!}{$\Mat{X}_{k} = \begin{bmatrix}
		\Vc{x}_{k,1} \\
		\Vc{x}_{k,2} \\
		\vdots \\
		\Vc{x}_{k,N_k}
	\end{bmatrix} = \begin{bmatrix}
		f_{\text{backbone}}(\mathcal{X}_{k,1}, \bm{\Theta}) \\
		f_{\text{backbone}}(\mathcal{X}_{k,2}, \bm{\Theta}) \\
		\vdots \\
		f_{\text{backbone}}(\mathcal{X}_{k,N_k}, \bm{\Theta})
	\end{bmatrix}
	\Mat{Y}_{k} = \begin{bmatrix}
		\onehot{y_{k,1}} \\
		\onehot{y_{k,2}} \\
		\vdots       \\
		\onehot{y_{k,N_k}}
	\end{bmatrix},$}
\end{equation}
where the embedding matrix $\Mat{X}_{k} \in \mathbb{R}^{N_{k} \times y_{\text{e}}}$,  and the label matrix $\Mat{Y}_{k} \in \mathbb{R}^{N_{k} \times C}$ has $C$ classes. The $\onehot{*}$ operator converts the index label  $y_{k,j}$ into a $C$-dimension one-hoted row vector. 

\par Subsequently, we approach the local client training with AL technique \cite{pil2001}. Specially, the target of the $k$-th client is to linearly map the extracted embeddings onto the one-hoted labels by minimizing the mean square error (MSE) loss function as follows.
\begin{equation}\label{eq_mes}
	\mathcal{L}(\Mat{W}_{k}) = \SFN{\Mat{Y}_{k} - \Mat{X}_{k}\Mat{W}_k},
\end{equation}
where $\lVert{*}\rVert_\mathrm{F}$ indicates the Frobenius norm. This leads to an optimal weight estimation $\Mat{\hat{W}}_{k}$, i.e.,
\begin{equation}\label{eq_w_k}
	\Mat{\hat{W}}_{k} = \argmin{\Mat{W}_{k}} \mathcal{L}(\Mat{W}_{k}) = \Mat{X}_{k}^{\pinv}\Mat{Y}_{k},
\end{equation}
where $\pinv$ denotes the Moore-Penrose (MP) inverse (also referred as generalized inverse or pseudoinverse) \cite{pil2001,brmp2021}.

The solution presented in \eqref{eq_w_k} optimally addresses the MSE loss function described in \eqref{eq_mes}, effectively establishing an LS-based AL solution for localized network learning.

\textbf{Why One-epoch Analytic Learning Works.} AL methods are generally effective for training shallow networks but face challenges when applied to deeper ones. This can be attributed to the fact that AL techniques are often designed as classifiers rather than end-to-end learning approaches. Despite this limitation, recent research has demonstrated that with a well-trained backbone, the AL performs adequately in various complex scenarios \cite{GACL_Zhuang_NeurIPS2024}. The practice of using a ``pre-trained backbone + downstream tasks'' has become increasingly common. This has allowed the one-epoch AL to thrive in various areas such as continual learning \cite{ACIL2022NeurIPS} and reinforcement learning \cite{alrl2024locality}. Hence, it could also be well incorporated in the individual client training.

Adopting AL is the key to enforcing the upcoming single-round aggregation (by deriving the AA law). The affine characteristic of linear regression in each client opens up new possibilities for exploration in FL. We provide a comprehensive explanation of such an exploration in later sections.

\subsection{Aggregation Stage: Absolute Aggregation Law}
In the aggregation stage, we introduce the Absolute Aggregation (AA) law, a key contribution in AFL. The AA law facilitates a single-round aggregation, i.e., the aggregation happens only once. Additionally, in scenarios where the AA law becomes suboptimal due to a large number of clients, we introduce a regularization intermediary (RI) process to restore its optimality.

The MP inverse partition \cite{Cline_GIPM_JSIAM1964} inspires our derivation, which is reformulated into Lemma \ref{corollary_decomp}.

\begin{lemma}\label{corollary_decomp}
	\par Let $\Mat{X} = \begin{bmatrix} \Mat{X}_{u} \\ \Mat{X}_{v} \end{bmatrix}$ with $\Mat{X}_{u}$ and $\Mat{X}_{v}$ having full column ranks, and $\Mat{X}$ follows a partition
	\begin{equation}
		\Mat{X}^{\pinv} = \begin{bmatrix} \Mat{\bar U} & \Mat{\bar V} \end{bmatrix},
	\end{equation}
	where
	\begin{equation*}\nonumber
 \resizebox{0.45\textwidth}{!}{$
		\begin{cases}
			\Mat{\bar U} = \Mat{X}_{u}^{\pinv} - \Mat{R}_{u}\Mat{C}_{v}\Mat{X}_{u}^{\pinv} + \Mat{R}_{u}\Mat{C}_{v}(\Mat{C}_{u} + \Mat{C}_{v})^{-1}\Mat{C}_{v}\Mat{X}_{u}^{\pinv} \\
			\Mat{\bar V} = \Mat{X}_{v}^{\pinv} - \Mat{R}_{v}\Mat{C}_{u}\Mat{X}_{v}^{\pinv} + \Mat{R}_{v}\Mat{C}_{u}(\Mat{C}_{u} + \Mat{C}_{v})^{-1}\Mat{C}_{u}\Mat{X}_{v}^{\pinv}
		\end{cases},$}
  \end{equation*}
  \begin{equation}
            \begin{cases}
			\Mat{C}_{u} = \Mat{X}_{u}^{\T}\Mat{X}_{u} \\
			\Mat{C}_{v} = \Mat{X}_{v}^{\T}\Mat{X}_{v}
		\end{cases}, \quad
            \begin{cases}
			\Mat{R}_{u} = \Mat{C}_{u}^{-1} \\
			\Mat{R}_{v} = \Mat{C}_{v}^{-1}
		\end{cases}.
	\end{equation}
\end{lemma}
\begin{proof}
	See Supplementary Materials A.
\end{proof}

Lemma \ref{corollary_decomp} points out that, a matrix's MP inverse (e.g., $\Mat{X}^{\pinv}$) can be computed using the inverse matrices of its block components (e.g., $\Mat{X}_{u}^{\pinv}$ and $\Mat{X}_{v}^{\pinv}$). This introduces possibilities for aggregating a weight $\Mat{W} = \Mat{X}^{\pinv}\Mat{Y}$ equally from manipulating constituent counterparts $\Mat{W}_{u} = \Mat{X}_{u}^{\pinv}\Mat{Y}_{u}$ and $\Mat{W}_{v} = \Mat{X}_{v}^{\pinv}\Mat{Y}_{v}$. That is, $\Mat{W} = f_{\text{agg}}(\Mat{W}_{u}, \Mat{W}_{v})$, i.e., a single-aggregation strategy.

Bearing the above intuition in mind, we are able to derive such a single-aggregation strategy in action. This is delivered in Theorem \ref{thm_agg}.
\begin{theorem}\label{thm_agg}
	\textup{\bf Absolute Aggregation Law:} Let $\Mat{\hat{W}}=\Mat{X}^{\pinv}\Mat{Y}$, where $\Mat{X} = \begin{bmatrix} \Mat{X}_{u} \\ \Mat{X}_{v} \end{bmatrix}$ and  $\Mat{Y} = \begin{bmatrix} \Mat{Y}_{u} \\ \Mat{Y}_{v} \end{bmatrix}$ with $\Mat{X}_{u}$ and $\Mat{X}_{v}$ having full column ranks. Let $\Mat{\hat{W}}_{u}=\Mat{X}_{u}^{\pinv}\Mat{Y}_{u}$, $\Mat{\hat{W}}_{v}=\Mat{X}_{v}^{\pinv}\Mat{Y}_{v}$, and we have
	\begin{equation}
		\Mat{W} = \Mat{\mathcal{W}}_{u}\Mat{W}_{u}+\Mat{\mathcal{W}}_{u}\Mat{W}_{v},
	\end{equation}
	where
		\begin{equation*}\label{eq_agg_1}\nonumber
			\begin{cases}
				\Mat{\mathcal{W}}_{u} = \Mat{I} - \Mat{R}_{u}\Mat{C}_{v} + \Mat{R}_{u}\Mat{C}_{v}(\Mat{C}_{u} + \Mat{C}_{v})^{-1}\Mat{C}_{v}\\
				\Mat{\mathcal{W}}_{v} = \Mat{I} - \Mat{R}_{v}\Mat{C}_{u} + \Mat{R}_{v}\Mat{C}_{u}(\Mat{C}_{u} + \Mat{C}_{v})^{-1}\Mat{C}_{u}
			\end{cases}
   \end{equation*}
   \begin{equation}
       \begin{cases}
				\Mat{C}_{u} = \Mat{X}_{u}^{\T}\Mat{X}_{u}\\
				\Mat{C}_{v} = \Mat{X}_{v}^{\T}\Mat{X}_{v}
			\end{cases}
			\begin{cases}
				\Mat{R}_{u} = \Mat{C}_{u}^{-1}\\
				\Mat{R}_{v} = \Mat{C}_{v}^{-1}
			\end{cases}.
   \end{equation}
			
\end{theorem}
\begin{proof}
	See Supplementary Materials B.
\end{proof}

The AA law, as stated in Theorem \ref{thm_agg}, provides a powerful insight. It establishes an intuition that we can aggregate two independently trained weights, such as $\Mat{W}_{u}$ and $\Mat{W}_{v}$, into their jointly trained counterpart $\Mat{W}$. This is achieved in an optimal way without any approximation or parameter tuning.

\textbf{Invariance to data partitioning.} To a certain extent, the achievement in Theorem \ref{thm_agg} attains the ultimate goal of FL, i.e., the \textit{equivalence} between weights trained in FL fashion and that trained on a centralized joint dataset. Traditionally, the FL aims to approximate or converge to the performance of the joint-trained model through multiple rounds of aggregation in a central server. However, the AA law provides a more direct path to this goal. It allows for an equivalence (not approximation or convergence) to manifest in a linear regression standpoint.

Supported by the AA law, the AFL achieves a level of performance that is on par with the joint-trained model, without the need for multiple rounds of aggregation. This direct equivalence could establish significant advancement in FL, as it simplifies the process and reduces the heavy computational overhead associated with multiple aggregation rounds.

Although the AA law in Theorem \ref{thm_agg} admits the absolute aggregation between two clients (i.e., $\Mat{\hat{W}}_{u}$ and $\Mat{\hat{W}}_{v}$), this pattern can be trivially broadcast to multi-client scenario. To elaborate, without loss of generality, we denote $\Mat{\hat W}_{\textup{agg},k-1}$ as the accumulated aggregation (AcAg) weight that has aggregated $k-1$ clients. By rewriting  \eqref{eq_agg_1}, the next aggregation with $\Mat{\hat W}_{k}$ ($i=k,\dots, K$) reads
\begin{equation}\label{eq_agg_mult}
	\Mat{\hat W}_{\textup{agg},k} = \Mat{\mathcal{W}}_{\textup{agg}}\Mat{\hat W}_{\textup{agg},k-1}       +     \Mat{\mathcal{W}}_{k} \Mat{\hat W}_{k}.
\end{equation} 
According to \eqref{eq_agg_1}, let $\Mat{C}_{u}\to\Mat{C}_{\textup{agg},k-1}$, $\Mat{C}_{v}\to\Mat{C}_{k}$, and we have $\Mat{C}_{\text{agg},k} = \Mat{C}_{\textup{agg},k-1} + \Mat{C}_{k}$. Hence, 
\begin{equation}\label{eq_weighting}
\begin{cases}
	\Mat{\mathcal{W}}_{\textup{agg}} = \Mat{I}-\Mat{C}_{\textup{agg},k-1}^{-1}\Mat{C}_{k}(\Mat{I}-\Mat{C}_{\textup{agg},k}^{-1}\Mat{C}_{k}), \\
	 \Mat{\mathcal{W}}_{k} = \Mat{I}-\Mat{C}_{k}^{-1}\Mat{C}_{\textup{agg},k-1}(\Mat{I}-\Mat{C}_{\textup{agg},k}^{-1}\Mat{C}_{\textup{agg},k-1}),
\end{cases}
\end{equation}
where 
\begin{equation}
\begin{cases}
	\Mat{C}_{\textup{agg},k} = \Mat{C}_{\textup{agg},k-1} + \Mat{C}_{k} = \sum_{i}^{k} \Mat{C}_{i} , \\
	\Mat{C}_{i} = \Mat{X}_{i}^{\T}\Mat{X}_{i}.
\end{cases}	
\end{equation}
As such, the joint-trained weight $\Mat{\hat W} = \Mat{\hat{W}}_{\textup{agg},k}$ is produced by aggregating among individual clients in a pair-wise manner. It is interesting to find that the optimal aggregation is in fact a linear combination between two matrices (e.g., $\Mat{\hat W}_{\textup{agg},k-1}$ and $\Mat{\hat W}_{k}$) weighted by $\Mat{\mathcal{W}}_{\textup{agg}}$ and $\Mat{\mathcal{W}}_{k}$ respectively.

Note that the aggregation does NOT necessarily follow a sequential index from $1$ to $K$. We can randomly sample an available client to aggregate with the AcAg weight. This is revealed by the fact that elements in the weighting matrices are somewhat interchangeable (e.g., see \eqref{eq_weighting}).

\subsection{RI Process: AA Law in Rank-deficient Scenario}
As indicated in Theorem \ref{thm_agg}, the equivalence in AA law relies on an assumption of a full-column rank in each client, e.g., $\Mat{X}_{k}$ having full-column rank. This may not hold in the large client number scenario where each client has limited data (e.g., $N_{k}<y_{\text{e}}$), rendering the full-column rank assumption invalid. To address this, we implement the AA law with an RI process. Specially, we include a regularization term as an intermediary during the local stage, and remove it after the aggregation stage.

To this end, we include an regularization term controlled by $\gamma$ in the objective function, i.e., 
\begin{equation}\label{eq_mes_r}
	\mathcal{L}(\Mat{W}_{k}^{\text{r}}) = \SFN{\Mat{Y}_{k} - \Mat{X}_{k}\Mat{W}_k^{\text{r}}} + \gamma \SFN{\Mat{W}_{k}^{\text{r}}}, 
\end{equation}
which rewrites the MP inverse based solution in \eqref{eq_w_k} into
\begin{equation}\label{eq_w_k_r}
	\Mat{\hat{W}}_{k}^{\text{r}} = \argmin{\Mat{W}_{k}^{\text{r}}} \mathcal{L}(\Mat{W}_{k}^{\text{r}}) = (\Mat{X}_{k}^{\T}\Mat{X}_{k} + \gamma \Mat{I})^{-1}\Mat{X}_{k}^{\T}\Mat{Y}_{k}.
\end{equation}
Such a solution does not suffer from rank-deficiency issues, as $\Mat{X}_{k}^{\T}\Mat{X}_{k} + \gamma \Mat{I}$ is positive-definite thereby a full-rank matrix.

%Subsequently, we proceed to the aggregation procedure in \eqref{eq_agg_mult} replacing $\Mat{\hat{W}}_{k}$, $\Mat{\hat{W}}_{\text{agg},k}$ with $\Mat{\hat{W}}_{k}^{\text{r}}$ and $\Mat{\hat{W}}_{\text{agg},k}^{\text{r}}$

During aggregation, we substitute $\Mat{\hat{W}}_{k}$ in \eqref{eq_w_k} with $\Mat{\hat{W}}_{k}^{\text{r}}$ using \eqref{eq_w_k_r}. This substitution would clearly result in deviations (i.e., $\Mat{\hat{W}}_{\text{agg},k}^{\text{r}}\ne\Mat{\hat{W}}_{\text{agg},k}$), which is depicted in Theorem \ref{thm_agg_r}.
\begin{theorem}\label{thm_agg_r}
	\textup{\bf RI-AA Law:} The relation between $\Mat{\hat{W}}_{\text{agg},k}^{\text{r}}$ and $\Mat{\hat{W}}_{\text{agg},k}$ follows
\begin{equation}
		\hat{\Mat{W}}_{\textup{agg},k}^{\textup{r}} = (\Mat{C}_{\textup{agg},k}^{\textup{r}})^{-1} \Mat{C}_{\textup{agg},k} \hat{\Mat{W}}_{\textup{agg},k},
\end{equation}
	where
	\begin{equation}\label{eq_agg_r}
 \resizebox{0.42\textwidth}{!}{
		$\Mat{C}_{\textup{agg},k}^{\textup{r}} = \Mat{C}_{\textup{agg},k} + k\gamma \Mat{I} = \sum_{i}^{k} \Mat{C}_{i}^{\textup{r}}, \quad\Mat{C}_{i}^{\textup{r}} = \Mat{X}_{i}^{\T}\Mat{X}_{i} +\gamma \Mat{I}.$}
	\end{equation}
	% As $\Mat{C}$ does not rely on the full-column rank assumption, we can obtain it via Theorem \ref{thm_agg}.
\end{theorem}
\begin{proof}
	See Supplementary Materials C.
\end{proof}

Theorem \ref{thm_agg_r} establishes the relation between $\Mat{\hat{W}}_{\text{agg},k}^{\text{r}}$ and $\Mat{\hat{W}}_{\text{agg},k}$, which is a one-to-one mapping, such that $\Mat{\hat{W}}_{\text{agg},k}$ can be restored by manipulating $\Mat{\hat{W}}_{\text{agg},k}^{\text{r}}$, i.e.,
\begin{align}\label{eq_rf}
	\nonumber \hat{\Mat{W}}_{\textup{agg},k} &=(\Mat{C}_{\textup{agg},k})^{-1} \Mat{C}_{\textup{agg},k}^{\textup{r}} \hat{\Mat{W}}_{\textup{agg},k}^{\textup{r}} \\ &= (\Mat{C}_{\textup{agg},k}^{\textup{r}} - k\gamma \Mat{I})^{-1} \Mat{C}_{\textup{agg},k}^{\textup{r}} \hat{\Mat{W}}_{\textup{agg},k}^{\textup{r}}.
\end{align}
That is, we are able to attain $\Mat{\hat{W}}_{\text{agg},k}$ by removing the impact of the regularization term $\gamma$ to counter the ill-conditioned constraint in the large client number scenario. The implementation of AFL is summarized in Algorithm \ref{alg:afl}.

\begin{algorithm}
    \caption{Analytic Federated Learning}\label{alg:afl}
    \vspace{-0.3em}
    % \begin{multicols}{2}
        \begin{algorithmic}
            \STATE {\bfseries Input:} $\dataset{D}_{k}, k = 0, ... ,K$, $\gamma$, and pre-trained backbone $\bm{\Theta}$. \\
            \STATE {\bfseries Server Executes:}
            \STATE 1. \textbf{for} each client $k$ \textbf{in parallel do}\\
            \STATE 2. \quad $\Mat{\hat W}_{k}^{\text{r}}, \Mat{C}_{k}^{\text{r}} \leftarrow$ \textbf{Local Stage}($k$, $\dataset{D}_{k}$, $\gamma$). \\
            \STATE 3. \textbf{end for} \\
            \STATE 4. $\Mat{\hat W} \leftarrow $ \textbf{Aggregation Stage}($\{\Mat{\hat{W}}_{k}^{\text{r}}, \Mat{C}_{k}^{\text{r}}, \gamma\}_{k=1}^{K}$). \\
            \STATE {\bfseries Local Stage:} client $k$ with $\dataset{D}_{k}$ and $\gamma$. \\
            \STATE 1. Get embedding and label matrices using \eqref{eq_emb_label}.\\
            \STATE 2. Obtain weight matrix $\Mat{\hat W}_{k}^{\text{r}}$ by \eqref{eq_w_k_r}.\\
            \STATE 3. Get $\Mat{C}_{k}^{\text{r}} =  \Mat{X}_{k}^{\T}\Mat{X}_{k}+\gamma\Mat{I}$. \\
            \STATE 4. Return $\Mat{\hat W}_{k}^{\text{r}}, \Mat{C}_{k}^{\text{r}}$.
            \STATE {\bf Aggregation Stage:} with $\{\Mat{\hat{W}}_{k}^{\text{r}}, \Mat{C}_{k}^{\text{r}}, \gamma\}_{k=1}^{K}$.
            \STATE 1. Initialize  $\Mat{\hat W}_{\textup{agg},0}^{\text{r}}=\Mat{0}$, $\Mat{C}_{\textup{agg},0}^{\text{r}}=\Mat{0}$.
            \STATE 2. \textbf{for} $k$ in range($K$):
            \STATE \quad \textcolor{blue}{i)} Aggregate $\Mat{\hat W}_{\textup{agg},k}^{\text{r}}$ with $\Mat{\hat W}_{k}^{\text{r}}$ using \eqref{eq_agg_mult}.
            \STATE \quad \textcolor{blue}{ii)} Update $\Mat{C}_{\textup{agg},k}^{\text{r}}=\Mat{C}_{\textup{agg},k-1}^{\text{r}}+\Mat{C}_{k}^{\text{r}}$.
            \STATE 3. \textbf{end for}.
            \STATE 4. Restore $\Mat{\hat W}=\Mat{\hat W}_{\text{agg},K}$ with $\Mat{\hat W}_{\text{agg},K}^{\text{r}}$ in \eqref{eq_rf}.
        \end{algorithmic}
    % \end{multicols}
\end{algorithm}

\textbf{Benefits of Adopting AL in AFL.} Inheriting from the AL technique, the AFL admits several merits over its gradient-based counterparts as follows. i) \textit{Fast training and convergence:} the analytical solutions allow AFL to finish the training and aggregation in one shot, exhibiting fast training and convergence. Also, the analytical solutions free the AFL from any convergence issue as no iterative-search based action is executed. ii) \textit{Low communication cost:} the single-round aggregation only requires a single communication between the clients and the server, which significantly reduces the communication cost. iii) \textit{Data heterogeneity invariance:} the invariance to data partitioning does not pose any constraint on data partition strategy. That is, the equivalence is hold across all possible data heterogeneous scenarios (e.g., see Section \ref{sec_data_partition}). iv) \textit{Client-number invariance:} for a complete dataset $\dataset{D}$ partitioned among $K$ clients (i.e., $\{\dataset{D}_k\}_{k=1}^{K}$), according to Theorem \ref{thm_agg} and \eqref{eq_agg_mult}, when the weights from all $K$ clients are aggregated, the resulting weight is identical to that trained on the full dataset $\dataset{D}$. To validate AA law with RI-process, we conduct an experiment on a dummy dataset and show the invariance (see Supplementary Materials D). 
% iv) \textit{Absolute convergence:} the analytical solutions free the AFL from any convergence issue as no iterative-search based action is executed. v) \textit{No hyperparameter:} according to Theorems \ref{thm_agg} and \ref{thm_agg_r}, no hyperparameter tuning process is needed (e.g., epoch number, learning rate, optimizer). To the authors' knowledge, the AFL is the first hyperparameter-free technique in FL history.

\textbf{An AL Branch of Federated Learning.} The AFL incorporates the AL technique and can be considered as an AL branch within the FL context. The AL and its recursive formulation have demonstrated remarkable adaptability in continual learning utilizing a well-trained backbone \cite{FOAL2024NeurIPS, GACL_Zhuang_NeurIPS2024}. In this case, this intuition has been extended to the FL field through non-trivial derivations.

\section{Experiments}
In this section, we provide extensive experiments to validate the proposed AFL, including comparison with FL state-of-the-arts and analysis under various settings. The training time and ablation study of regularization are also investigated.

\begin{table*}[t!]
	\centering  
	\caption{The top-1 accuracy (\%) of compared methods under two non-IID settings. Settings controlled by $\alpha$ and $s$ are NIID-1 and  NIID-2 respectively. The data is reported as average and standard deviation after 3 runs. Results in \textbf{bold} are the best within the compared methods in the same setting.}
	\resizebox{1\textwidth}{!}{
		\begin{tabular}{llcccccccc}
			\toprule
			\textbf{Dataset} & \textbf{Setting} &  \textbf{FedAvg} & \textbf{FedProx} & \textbf{MOON} & \textbf{FedGen} & \textbf{FedDyn} & \textbf{FedNTD} & \textbf{FedDisco} & \textbf{AFL} \\ \hline
			\multirow{4}{*}{CIFAR-10} & $\alpha=0.1$ & $64.02_{\pm{0.18}}$ & $64.07_{\pm{0.08}}$ & $63.84_{\pm{0.03}}$ & $64.14_{\pm{0.24}}$ & $64.77_{\pm{0.11}}$ & $64.64_{\pm{0.02}}$ & $63.83_{\pm{0.08}}$ & $\textbf{80.75}_{\pm{0.00}}$ \\
			& $\alpha=0.05$ & $60.52_{\pm{0.39}}$ & $60.39_{\pm{0.09}}$ & $60.28_{\pm{0.17}}$ & $60.65_{\pm{0.19}}$ & $60.35_{\pm{0.54}}$ & $61.16_{\pm{0.33}}$ & $59.90_{\pm{0.05}}$ & $\textbf{80.75}_{\pm{0.00}}$ \\ \cline{2-10}
			& $s=4$ & $68.47_{\pm{0.13}}$ & $68.46_{\pm{0.08}}$ & $68.47_{\pm{0.15}}$ & $68.24_{\pm{0.28}}$  & $73.50_{\pm{0.11}}$ & $70.24_{\pm{0.11}}$ & $65.04_{\pm{0.11}}$ & $\textbf{80.75}_{\pm{0.00}}$\\ 
			& $s=2$ & $57.81_{\pm{0.03}}$ & $57.61_{\pm{0.12}}$ & $57.72_{\pm{0.15}}$ & $57.02_{\pm{0.18}}$ & $64.07_{\pm{0.09}}$ & $58.77_{\pm{0.18}}$ & $58.78_{\pm{0.02}}$& $\textbf{80.75}_{\pm{0.00}}$ \\  \hline
			\multirow{4}{*}{CIFAR-100} 
			& $\alpha=0.1$ & $56.62_{\pm{0.12}}$ & $56.45_{\pm{0.22}}$ & $56.58_{\pm{0.02}}$ & $56.48_{\pm{0.17}}$  & $57.55_{\pm{0.08}}$& $56.60_{\pm{0.14}}$ & $55.79_{\pm{0.04}}$ & $\textbf{58.56}_{\pm{0.00}}$\\ 
            & $\alpha=0.01$ & $32.99_{\pm{0.20}}$ & $33.37_{\pm{0.09}}$ & $33.34_{\pm{0.11}}$ & $33.09_{\pm{0.09}}$ & $36.12_{\pm{0.08}}$ & $32.59_{\pm{0.21}}$ & $25.72_{\pm{0.08}}$ & $\textbf{58.56}_{\pm{0.00}}$ \\ \cline{2-10}
			& $s=10$ & $55.76_{\pm{0.13}}$ & $55.80_{\pm{0.16}}$ & $55.70_{\pm{0.25}}$ & $60.93_{\pm{0.17}}$ & $\textbf{61.09}_{\pm{0.09}}$ & $54.69_{\pm{0.15}}$ & $54.65_{\pm{0.09}}$ & $58.56_{\pm{0.00}}$ \\ 
			& $s=5$ & $48.33_{\pm{0.15}}$ & $48.29_{\pm{0.14}}$ & $48.34_{\pm{0.19}}$ & $48.12_{\pm{0.06}}$ & $\textbf{59.34}_{\pm{0.11}}$ & $47.00_{\pm{0.19}}$ & $45.86_{\pm{0.18}}$ &  $58.56_{\pm{0.00}}$ \\  \hline
			\multirow{4}{*}{Tiny-ImageNet} & $\alpha=0.1$ & $46.04_{\pm{0.27}}$ & $46.47_{\pm{0.23}}$ & $46.21_{\pm{0.14}}$ & $46.27_{\pm{0.14}}$ & $47.72_{\pm{0.22}}$ & $46.17_{\pm{0.16}}$ & $47.48_{\pm{0.06}}$ & $\textbf{54.67}_{\pm{0.00}}$\\
			& $\alpha=0.01$ & $32.63_{\pm{0.19}}$ & $32.26_{\pm{0.14}}$ &$32.38_{\pm{0.20}}$ & $32.33_{\pm{0.14}}$ &$35.19_{\pm{0.06}}$ & $31.86_{\pm{0.44}}$ & $27.15_{\pm{0.10}}$&    $\textbf{54.67}_{\pm{0.00}}$\\  \cline{2-10}
			& $s=10$ & $39.06_{\pm{0.26}}$ & $38.97_{\pm{0.23}}$ & $38.79_{\pm{0.14}}$ & $38.82_{\pm{0.16}}$ & $41.36_{\pm{0.06}}$ & $37.55_{\pm{0.09}}$ & $38.86_{\pm{0.12}}$ & $\textbf{54.67}_{\pm{0.00}}$\\  
			& $s=5$ & $29.66_{\pm{0.19}}$  & $29.17_{\pm{0.16}}$ & $29.24_{\pm{0.30}}$ & $29.37_{\pm{0.25}}$  & $35.18_{\pm{0.18}}$ &  $29.01_{\pm{0.14}}$&$27.72_{\pm{0.18}}$  &  $\textbf{54.67}_{\pm{0.00}}$\\ 
			\bottomrule
		\end{tabular}}
	\label{table:compare}
\end{table*}

\subsection{Comparison with FL Techniques} \label{sec:compared}
We conduct comparison with FL state-of-the-arts, including FedAvg \cite{FedAvg2017ICML}, FedProx \cite{FedProx2020MLSYS}, MOON \cite{MOON2021CVPR}, FedGen \cite{FedGen2021ICML}, FedDyn \cite{FedDyn2021ICLR}, FedNTD \cite{FedNTD2022NeurIPS} and FedDisco \cite{FedDisco2023ICML} under various non-IID settings. 

\textbf{Dataset and Model.} We validate the baselines and our proposed AFL in 3 popular benchmark datasets in FL: CIFAR-10 \cite{CIFAR2009}, CIFAR-100 \cite{CIFAR2009} and Tiny-ImageNet \cite{tiny-imagenet}. For all datasets, we use a ResNet-18 \cite{resnet2016CVPR} pretrained on ImageNet-1k \cite{imagnet2009CVPR} as backbone. We freeze the backbones in all FL methods. 
% Please refer to Supplementary Materials E to see the results of AFL with different backbones.

% The CIFAR-10 contains of 60000 images in which 50000 images is used for training and 10000 images is for testing equally distributed in 10 classes. CIFAR-100 has the same number of image as CIFAR-10 for training and testing but has 100 classes. The Tiny-ImageNet dataset contains 100000 images of 200 classes (500 for each class) for training and 50 images for validation and testing of each class. 

\textbf{Data Partition.} For simulating Non-IID scenarios in FL, we specify two Non-IID data partition methods including Latent Dirichlet Allocation \cite{FedDF2020NEURIPS} (LDA, denoted as NIID-1) and Sharding \cite{FedDF2020NEURIPS} (denoted as NIID-2). In the LDA setting, data assigned to each clients is forced to satisfy the Dirichlet distribution, and degree of the data heterogeneity is controlled by parameter $\alpha$. Smaller $\alpha$ leads to a more heterogeneous data distribution. In the Sharding strategy, the data is sorted by labels and divided into same-sized shards, and $s$ control the heterogeneity, i.e. the number of shards per client. When $s$ takes a smaller value, the data is more heterogeneous. We choose $\alpha = 0.1, 0.01$ and $s = 10, 4$ for CIFAR-100 and Tiny-ImageNet. For CIFAR-10, $\alpha$ is set to $0.1$ and $0.01$, and $s$ is set to $4$ and $2$. Most existing methods are validated on data partition of $\alpha = 0.3$ to $1.0$ and $s = 10$ \cite{FedNTD2022NeurIPS, FedLAW2023ICML}. Here we provide more challenging settings to validate the robustness under extremely heterogeneous cases.

\textbf{Implementation Details.} In all the experiments, we use 100 clients for each method and use the same partitioned dataset within experiments of the same data setting. We implement the AFL with a $\gamma=1$ RI process (any $\gamma$ would suffice, see the ablation study). Each experiment setting is run 3 times and the mean and standard deviation of the best top-1 classification accuracy during training are reported. The implementation details of gradient-based compared methods can be found in Supplementary Materials E.

% Since AFL obtains the analytic solution and will get the same results with the same frozen backbone and data partition, we only run once and report the accuracy on test set. Since FL methods may get far away from the global optimum with larger local epochs with data heterogeneity and the selection of clients can intensify the heterogeneity \cite{FedNIID_survey2021ICDE}, we set t to reduce the effect of data NonIID on compared methods

\textbf{Experimental Results.} We report the results of the compared methods under the setting of NIID-1 and NIID-2 in Table \ref{table:compare}. As shown in the table, except for slightly weaker results than those of FedDyn in the NIID-2 setting, the AFL obtains very competitive performance compared with other methods across various settings. The degree of data heterogeneity does not at all affect the AFL. For instance, the accuracy remains 80.75\% on CIFAR-10 for various NIID-1 and NIID-2 settings. Although slight differences could occur among various settings, it barely impacts the classification accuracy (an AL property indicated in \cite{brmp2021}). The same pattern repeats on CIFAR-100 and Tiny-ImageNet. Uniquely, the AFL obtains identical results for all 3 repeated runs, i.e., the standard deviations are zeros! This is because the AFL does not introduce any stochastic element, so the repeated computation in each run is naturally equivalent to one another, hence the zero standard deviation. 

Notably, when introducing a pre-trained backbone, the compared methods yield similar results and even FedAvg can become a competitive baseline. The reason could be that, when incorporating a pret-trained backbone, the FL process can be more stable with a better start point and methods stabilizing the FL process could be less effective. This phenomena has also been witnessed in other FL studies \cite{FLPretrained2023ICLR, FedPR2023CVPR}. However, other FL methods still experience performance reductions under severe non-IID scenarios. For example, FedDyn performs relatively well (e.g., 57.55\%) under NIID-1 ($\alpha=0.1$) on CIFAR-100 but undergoes a performance degradation (to 36.12\%) when $\alpha=0.01$. This pattern is rather consistent in other compared methods, such as FedAvg ($56.62\%\to32.99\%$), FedProx ($56.45\%\to33.37\%$), and MOON ($56.58\%\to33.34\%)$, and is also true across all datasets. The performance distributions regarding NIID-2 for these compared methods resemble those in NIID-1, where smaller $s$ values invite performance degradation among existing FL counterparts. For instance, the FedDyn exhibits $73.50\%\to64.07\%$ for $s=4\to2$ on CIFAR-10 while the AFL obtains competitive and identical results (e.g., $80.75\%$).

%        \begin{tabular}{lcccccccccccccccccc}
	%    \toprule
	% \textbf{Dataset} && \textbf{Setting} &&  \textbf{FedAvg} && \textbf{FedProx} && \textbf{MOON} && \textbf{FedGen} && \textbf{FedDyn} && \textbf{FedNTD} && \textbf{FedDisco} && \textbf{AFL} \\ \hline
	%    \multirow{4}{*}{CIFAR-10} && $\alpha=0.1$ &&  &&  &&  &&  &&  &&  && && \\
	%    && $\alpha=0.05$ &&  &&  &&  &&  &&  &&  && && \\ 
	%    && $s=4$ &&  &&  &&  &&  &&  &&  && && \\ 
	%    && $s=2$ &&  &&  &&  &&  &&  &&  && && \\  \hline
	%    \multirow{4}{*}{CIFAR-100} && $\alpha=0.1$ &&  &&  &&  &&  &&  &&  && && \\
	%    && $\alpha=0.01$ &&  &&  &&  &&  &&  &&  && && \\ 
	%    && $s=10$ &&  &&  &&  &&  &&  &&  && && \\ 
	%    && $s=5$ &&  &&  &&  &&  &&  &&  && && \\  \hline
	%    \multirow{4}{*}{Tiny-ImageNet} && $\alpha=0.1$ &&  &&  &&  &&  &&  &&  && && \\
	%    && $\alpha=0.01$ &&  &&  &&  &&  &&  &&  && && \\ 
	%    && $s=10$ &&  &&  &&  &&  &&  &&  && && \\ 
	%    && $s=5$ &&  &&  &&  &&  &&  &&  && && \\ 
	%    \bottomrule
	%    \end{tabular}
\subsection{Analysis on Data Partition}\label{sec_data_partition}
Here we provide broaden non-IID partitions to demonstrate AFL's invariance to data partitioning. This includes varying the client number and the non-IID degree. We also provide the IID partition results. 

\textbf{Client-number Invariance.} We compare our AFL and the FedAvg under NIID-1 setting on CIFAR-100 and Tiny-ImageNet with $\alpha=0.1$, and vary the number of clients from 100 to 500 and 1000. The results are shown in Figure \ref{fig:num_client}. We observe that the AFL keeps an identical performance when scaling the number of clients, while the FedAvg experiences a performance decline along the increasing number (e.g, 56.57\%$\to$41.01\% for $K=100\to1000$ on CIFAR-100). This provides a strong evidence to support the invariance to data partitioning in our AFL. It also showcases the capability of pushing the AFL to large-scale client training scenario without any performance compromise.

\begin{figure}[h]
    \centering
    \includegraphics[width=\linewidth]{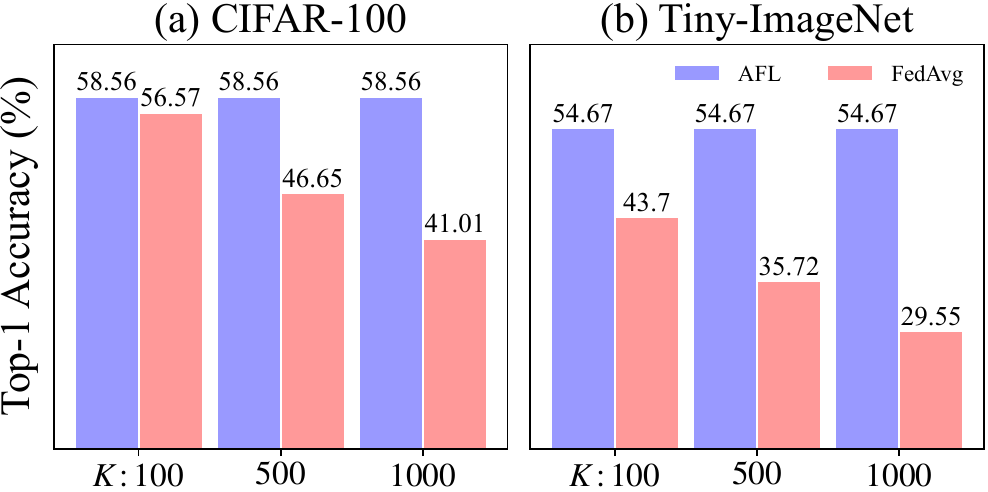}
    \caption{Accuracy over various number of clients.}
    \label{fig:num_client}
\end{figure}

\textbf{Data Heterogeneity Invariance.} Here, we fix the client number to 100 and partition the CIFAR-100 under the setting of NIID-1 with $\alpha = {0.005, 0.01, 0.1, 1}$, including the IID setting as well. We report the results of AFL and FedAvg in Table \ref{table:noniid}. The FedAvg suffers from more accuracy losses (e.g., $57.72\%\to24.74\%$ for $\alpha=0.1\to0.005$) as the data heterogeneity grows higher. Under the IID partition, the FedAvg receives its best performance (i.e., $57.89\%$), which is still less competitive than our AFL (i.e., $58.56\%$). On the other hand, AFL obtains identical results (i.e., $58.56\%$) across various settings, including non-IID and IID ones. This is another strong proof of the AA law indicating the weight-invariant property of AFL. Our AFL is invariant to any degree of data heterogeneity, leading to unchanged performance in all possible data heterogeneous partition scenarios, even in extreme data heterogeneous cases (e.g., $\alpha=0.005$). 

\begin{table}[!h]
\footnotesize
	  
	\captionof{table}{The top-1 classification accuracy (\%) of AFL and FedAvg under different data heterogeneity.}
    % \captionof{table}{The top-1 classification accuracy (\%) of AFL and FedAvg under different data heterogeneity.}
	% \resizebox{1\textwidth}{!}{
    \begin{tabular}{lccccc} \toprule
		Acc. (\%)& $\alpha=0.005$ & $\alpha=0.01$ & $\alpha=0.1$ & $\alpha=1$ & IID \\ \hline
		FedAvg  & 24.74 & 33.09 & 56.57 & 57.72 & 57.89\\ \hline
		AFL  & \textbf{58.56} & \textbf{58.56} & \textbf{58.56} & \textbf{58.56} & \textbf{58.56} \\ \bottomrule
	\end{tabular}
 % }
 \centering
	\label{table:noniid}
\end{table}

\subsection{Training Efficiency}
\textbf{Fast Training with Single-round Aggregation.} We plot the training evolution curves of accuracy on CIFAR-100 and Tiny-ImageNet in Figure \ref{fig:time} and report the execution time for each method in the legend bars. Compared FL methods take 60s to 100s on CIFAR-100 (100s to 160s on Tiny-ImageNet) to complete an aggregation round, leading to a total training time of 30,000s to 50,000 (50,000s to 80,000s). AFL, however, spends 236.61s on CIFAR-100 and 349.50s on Tiny-ImageNet, achieving approximately 150$\times$-200$\times$ speedups over its FL counterparts due to only one aggregation.

\begin{figure}[h]
    \centering
    \includegraphics[width=0.95\linewidth]{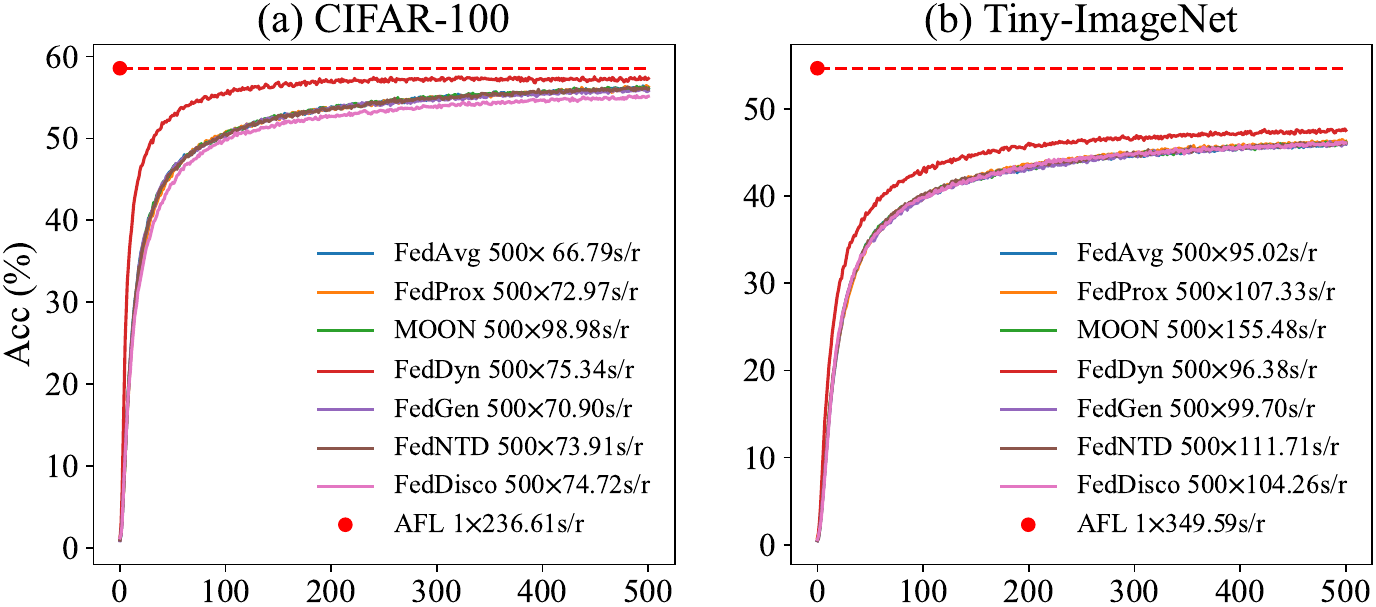}
    \caption{Accuracy curves with communication rounds. Average training time is reported in the legends.}
    \label{fig:time}
\end{figure}

\subsection{Ablation Study of RI Process}
Here, we conduct an ablation study regarding the RI process by reporting accuracies of AFL with $\alpha = 0.1$ and $K = 100, 500, 1000$ under different values of $\gamma$. The results without and with the RI process are provided in Table \ref{table:ablation}. When $\gamma=0$ (i.e., no regularization involved), the AFL stops working with $K=500$ and 1000 due to the ill-conditioned matrix scenario (e.g., $N_{k}<y_{\text{e}}$). Such an ill-conditioned case is avoided by introducing $\gamma$. However, the lack of RI process (see left columns in Table \ref{table:ablation}) could lead to accuracy loss. For instance, for $\gamma=100$, the AFL could suffer a loss of 9\% (i.e., $58.56\%\to49.62\%$). This is the result of regularization accumulation (see \eqref{eq_agg_r}). With the RI process, the AFL obtains an identical result across various $\gamma$ values. More importantly, this demonstrates that adopting the RI avoids the need to find proper $\gamma$ values. That is, the regularization is a removable intermediary, not a hyperparameter that requires tuning. 

\begin{table}[!h]
    \centering
	\caption{Ablation study of RI under various $\gamma$ and $K$. The left/right results are performance w/o and w/ the RI process in \eqref{eq_rf}.}\label{table:ablation}
 \resizebox{0.48\textwidth}{!}{
\footnotesize
  \begin{tabular}{lccccc} \toprule
		Acc.(\%)& $\gamma = 0 $& $\gamma = 0.1$& $\gamma = 1$& $\gamma = 10$& $\gamma = 100$ \\ \hline
		K=100  & 58.56 $|$ N/A& 58.54 $|$ 58.56& 58.51 $|$ 58.56& 58.15 $|$ 58.56& 55.77 $|$ 58.56 \\ \hline
        K=500  & ~1.11 $|$ N/A&58.52 $|$ 58.56& 58.30 $|$ 58.56& 56.72 $|$ 58.56 & 51.77 $|$ 58.56 \\ \hline
		K=1000  & ~0.75 $|$ N/A&58.51 $|$ 58.56& 58.15 $|$ 58.56& 55.77 $|$ 58.56 & 49.62 $|$ 58.56 \\ \bottomrule
        % K=5000  &  $|$ N/A  &58.30 $|$ 58.56 & 56.72 $|$ 58.56 & 51.77 $|$ 58.56  & \\ \bottomrule
	\end{tabular}}
\end{table}

\subsection{Validation with Different Backbones}
To explore the effect of different backbones used in the AFL, we extend the experiments with VGG11 \cite{VGG2015ICLR} and ViT-B-16 \cite{ViT2021ICLR}. All these backbone are pre-trained in ImageNet-1k and we conduct the experiments under the same setting in Section 4.2. Due to the invariance to data partitioning, we only report one result in single dataset. As shown in the Table \ref{tab:backbones}, with different pre-trained backbones, the AFL can all obtain competitive results. 
% The results with ViT-B-16 shows best performance within these three backbones. That's because the ViTs have strong transferrable features in the downstream tasks. 

\begin{table}[h]
\footnotesize
\caption{The results of top-1 accuracy in \% of the AFL with different backbones including ResNet-18, VGG11 and ViT-B-16.}
    \centering
    \begin{tabular}{cccc}\toprule
        Acc.($\%$) & CIFAR-10& CIFAR-100& Tiny-ImageNet \\ \hline
        ResNet-18 & 80.75& 58.56& 54.67 \\ \hline
        VGG11 & 82.72& 60.43& 54.73 \\ \hline
        ViT-B-16 & \textbf{93.92}& \textbf{75.45}& \textbf{82.02} \\
         \bottomrule
    \end{tabular}
    
    \label{tab:backbones}
\end{table}
\section{Limitations and Future Work}
\textbf{Utilizing Pre-trained Backbone.} The AFL approach is both facilitated and constrained by the requirement of having a well-trained feature extractor. However, this limitation has been significantly mitigated by the emergence of reusing pre-trained models for new tasks. This ``pre-trained backbone + downstream task'' paradigm has become a standard practice in numerous deep learning domains, offering improved generalization and reduced computational costs. FL can further enhance this paradigm and we validate that collaboration can still be beneficial with pre-trained backbones in Supplementary Materials F. The proposal of the AFL aligns with these recent research trends, making it a sensible FL advancement. 
% Also, this limitation inspires future efforts to adjust the pre-trained backbone and further enhance the performance of AFL. 

\textbf{Partially Participating and Stragglers.} The AFL formulates a single-round aggregation for FL systems, promoting rapid convergence and reducing communication overhead. However, challenges arise when clients engage partially or when stragglers impede progress. Since clients can only contribute to the aggregation after finishing local computations, the AFL needs to wait for all the clients. This potentially hampers the AFL's overall efficiency and inspires us to further refine the AFL to address these issues.

\textbf{Linear Assumptions of AFL.}  The AFL is established upon linear classifiers and may be less effective with non-linear data distribution. To address this, AFL can incorporate non-linear projections including non-linear activations or kernel functions. Also, for multi-layer model, AFL can formulate local least-square problem at each layer by label projection \cite{cpnet2021}. These techniques have been utilized in various AL-based work \cite{RAIL2024NeurIPS} and the AA law holds theoretically. We will conduct a further exploration in future.  

\section{Conclusion}
In this paper, we introduce a gradient-free FL framework named analytic federated learning (AFL). The AFL unveils analytical solutions both in the local client training stage and the aggregation stage. This leads to one-epoch local training, single-round aggregation, and fast convergence. In particular, the single-round aggregation property is theoretically supported and proved by the well-formulated AA law. Additionally, by introducing the RI process, we re-establish the AFL's optimality which could be compromised in the scenario of rank-deficient with normally a large number of clients. The AFL demonstrates its invariance to data partitioning, a property that allows several appealing FL characteristics such as data heterogeneity invariance and client-number invariance. These characteristics are empirically validated through experiments across various settings, where the AFL achieves a consistent and competitive performance.

\section*{Acknowledgment}
% This research was supported by the National Natural Science Foundation of China (62306117), the Guangzhou Basic and Applied Basic Research Foundation (2024A04J3681, 2023A04J1687), the GJYC program of Guangzhou (2024D03J0005), the South China University of Technology-TCL Technology Innovation Fund, the Fundamental Research Funds for the Central Universities (2023ZYGXZR023, 2024ZYGXZR074), the Guangdong Basic and Applied Basic Research Foundation (2024A1515010220), the CAAI-MindSpore Open Fund developed on Openl Community, the Shenzhen Fundamental Research Program (JCYJ20230807091809020), and Shenzhen Science and Technology Plan (Grant No. JCYJ20210324123802006).
This research was supported by the National Natural Science Foundation of China (62306117, 62406114), the Guangzhou Basic and Applied Basic Research Foundation (2024A04J3681, 2023A04J1687), GJYC program of Guangzhou (2024D03J0005), the National Key R \& D Project from Minister of Science and Technology (2024YFA1211500), and the Fundamental Research Funds for the Central Universities (2024ZYGXZR074).
{
    \small
    \bibliographystyle{ieeenat_fullname}
    \bibliography{afl}
}


\clearpage
\setcounter{page}{1}
\maketitlesupplementary
\renewcommand\theequation{A.\arabic{equation}}
\renewcommand\thetable{A.\arabic{table}}
\setcounter{equation}{0}
\setcounter{table}{0}
\appendix
% \begin{appendices}
\renewcommand{\thesection}{\Alph{section}}
\section{Proof of Lemma 1}\label{app_corr1}
\begin{proof}
We prove this Lemma mainly based on the existing MP inverse partition result \cite{Cline_GIPM_JSIAM1964} as follows.	
	
In \cite{Cline_GIPM_JSIAM1964}, it has been demonstrated that the MP inverse of any matrix, $\Mat{\mathcal{A}} = \begin{bmatrix} \Mat{U} & \Mat{V} \end{bmatrix}$ can be written as
\begin{equation}\nonumber
\Mat{\mathcal{A}}^{\pinv} = \begin{bmatrix}
		\Mat{U} & \Mat{V}
	\end{bmatrix}^{\pinv} =
\end{equation}
\begin{equation}
 \label{eq_app_Cline1}
 \resizebox{0.45 \textwidth}{!}{
 $ \begin{bmatrix}
		\Mat{U}^{\pinv} - \Mat{U}^{\pinv}\Mat{V}\Mat{C}^{\pinv} - \Mat{U}^{\pinv}\Mat{V}(\Mat{I} - \Mat{C}^{\pinv}\Mat{C})\Mat{K}\Mat{V}^{\cT}\Mat{U}^{\pinv\cT}\Mat{U}^{\pinv}(\Mat{I} - \Mat{V}\Mat{C}^{\pinv}) \\
		\Mat{V}^{\pinv} - \Mat{V}^{\pinv}\Mat{U}\Mat{\tilde{C}}^{\pinv} - \Mat{V}^{\pinv}\Mat{U}(\Mat{I} - \Mat{\tilde{C}}^{\pinv}\Mat{\tilde{C}})\Mat{\tilde{K}}\Mat{U}^{\cT}\Mat{V}^{\pinv\cT}\Mat{V}^{\pinv}(\Mat{I} - \Mat{U}\Mat{\tilde{C}}^{\pinv})
	\end{bmatrix}$},
\end{equation}
where
\begin{equation*}\nonumber
	\begin{cases}
		\Mat{C}         = (\Mat{I} - \Mat{U}\Mat{U}^{\pinv})\Mat{V} \\
		\Mat{\tilde{C}} = (\Mat{I} - \Mat{V}\Mat{V}^{\pinv})\Mat{U} \\
	\end{cases}, 
\end{equation*}
 \begin{equation}\label{eq_app_Cline2}
	\begin{cases}
		\Mat{K}         = \left[ \Mat{I} + (\Mat{I} - \Mat{C}^{\pinv}\Mat{C})\Mat{V}^{\cT}\Mat{U}^{\pinv\cT}\Mat{U}^{\pinv}\Mat{V}(\Mat{I} - \Mat{C}^{\pinv}\Mat{C}) \right]^{-1} \\
		\Mat{\tilde{K}} = \left[ \Mat{I} + (\Mat{I} - \Mat{C}^{\pinv}\Mat{C})\Mat{U}^{\cT}\Mat{V}^{\pinv\cT}\Mat{V}^{\pinv}\Mat{U}(\Mat{I} - \Mat{C}^{\pinv}\Mat{C}) \right]^{-1} \\
	\end{cases}.
\end{equation}

In the case of $\Mat{X} = \begin{bmatrix} \Mat{X}_{u} \\ \Mat{X}_{v} \end{bmatrix}$ with only real numbers, we substitute $\Mat{U}$ with $\Mat{X}_{u}^{\T}$, $\Mat{V}$ with $\Mat{X}_{v}^{\T}$. This rewrites \eqref{eq_app_Cline1}, \eqref{eq_app_Cline2} into
\begin{equation*} \nonumber
	\Mat{X}^{\pinv} = \begin{bmatrix}
		\Mat{X}_{u} \\
		\Mat{X}_{v}
	\end{bmatrix}^{\pinv}  =
 \begin{bmatrix}
     \Mat{X}_{u}^{\pinv\T} - \Mat{X}_{u}^{\pinv\T}\Mat{X}_{v}^{\T}\Mat{C}^{\pinv} \\
     \Mat{X}_{v}^{\pinv\T} - \Mat{X}_{v}^{\pinv\T}\Mat{X}_{u}^{\T}\Mat{\tilde{C}}^{\pinv}
     \end{bmatrix}^{\T} - 
 \end{equation*}
 \begin{equation}
  \begin{bmatrix}
      \Mat{X}_{u}^{\pinv\T}\Mat{X}_{v}^{\T}(\Mat{I} - \Mat{C}^{\pinv}\Mat{C})\Mat{K}\Mat{X}_{v}\Mat{X}_{u}^{\pinv}\Mat{X}_{u}^{\pinv\T}(\Mat{I} - \Mat{X}_{v}^{\T}\Mat{C}^{\pinv})
		 \\
	 \Mat{X}_{v}^{\pinv\T}\Mat{X}_{u}^{\T}(\Mat{I} - \Mat{\tilde{C}}^{\pinv}\Mat{\tilde{C}})\Mat{\tilde{K}}\Mat{X}_{u}\Mat{X}_{v}^{\pinv}\Mat{X}_{v}^{\pinv\T}(\Mat{I} - \Mat{X}_{u}^{\T}\Mat{\tilde{C}}^{\pinv})
	\end{bmatrix}^{\T},
\end{equation}
where
\begin{equation*} \nonumber
	\begin{cases}
		\Mat{C}         = (\Mat{I} - \Mat{X}_{u}^{\T}\Mat{X}_{u}^{\pinv\T})\Mat{X}_{v}^{\T} \\
		\Mat{\tilde{C}} = (\Mat{I} - \Mat{X}_{v}^{\T}\Mat{X}_{v}^{\pinv\T})\Mat{X}_{u}^{\T} \\
	\end{cases},
\end{equation*}
\begin{equation}
\resizebox{0.45\textwidth}{!}{$
	\begin{cases}
		\Mat{K}         = \left[ \Mat{I} + (\Mat{I} - \Mat{C}^{\pinv}\Mat{C})\Mat{X}_{v}\Mat{X}_{u}^{\pinv}\Mat{X}_{u}^{\pinv\T}\Mat{X}_{v}^{\T}(\Mat{I} - \Mat{C}^{\pinv}\Mat{C}) \right]^{-1} \\
		\Mat{\tilde{K}} = \left[ \Mat{I} + (\Mat{I} - \Mat{C}^{\pinv}\Mat{C})\Mat{X}_{u}\Mat{X}_{v}^{\pinv}\Mat{X}_{v}^{\pinv\T}\Mat{X}_{u}^{\T}(\Mat{I} - \Mat{C}^{\pinv}\Mat{C}) \right]^{-1} \\
	\end{cases}.
 $}
\end{equation} 

As $\Mat{X}_{u}$ and $\Mat{X}_{v}$ are of full column ranks, we obtain an alternative formulation of the MP inverse, i.e., 
\begin{equation}
	\Mat{X}_{u}^{\pinv} = (\Mat{X}_{u}^{\T}\Mat{X}_{u})^{-1}\Mat{X}_{u}^{\T}, \quad \Mat{X}_{v}^{\pinv} = (\Mat{X}_{v}^{\T}\Mat{X}_{v})^{-1}\Mat{X}_{v}^{\T}.
\end{equation}
Hence we have
\begin{align} \nonumber
	\Mat{C} &= (\Mat{I} - \Mat{X}_{u}^{\T}\Mat{X}_{u}^{\pinv\T})\Mat{X}_{v}^{\T} \\ & = (\Mat{I} - \Mat{X}_{u}^{\T}\Mat{X}_{u}(\Mat{X}_{u}^{\T}\Mat{X}_{u})^{-1})\Mat{X}_{v}^{\T} = \Mat{0}.
\end{align}
Similarly, 
\begin{align}\nonumber
	\Mat{\tilde{C}} &= (\Mat{I} - \Mat{X}_{v}^{\T}\Mat{X}_{v}^{\pinv\T})\Mat{X}_{u}^{\T} \\ & = (\Mat{I} - \Mat{X}_{v}^{\T}\Mat{X}_{v}(\Mat{X}_{v}^{\T}\Mat{X}_{v})^{-1})\Mat{X}_{u}^{\T} = \Mat{0}.
\end{align}
This simplifies $\Mat{K}$ and $\Mat{\tilde{ K}}$ as
\begin{equation}
	\begin{cases}
		\Mat{K}         = (\Mat{I} + \Mat{X}_{v}\Mat{X}_{u}^{\pinv}\Mat{X}_{u}^{\pinv\T}\Mat{X}_{v}^{\T})^{-1} \\
		\Mat{\tilde{K}} = (\Mat{I} + \Mat{X}_{u}\Mat{X}_{v}^{\pinv}\Mat{X}_{v}^{\pinv\T}\Mat{X}_{u}^{\T})^{-1}
	\end{cases}.
\end{equation}
According to the Woodbury Matrix Identity, i.e., for conformable matrices $\Mat{A} \in \mathbb{R}^{n \times n}$, $\Mat{B} \in \mathbb{R}^{n \times m}$, $\Mat{E} \in \mathbb{R}^{m \times m}$, and $\Mat{D} \in \mathbb{R}^{m \times n}$, 
	\begin{equation}\label{eq_woodbury}
 \resizebox{0.45\textwidth}{!}{$
		(\Mat{A} + \Mat{B}\Mat{E}\Mat{D})^{-1} = \Mat{A}^{-1} - \Mat{A}^{-1}\Mat{B}(\Mat{E}^{-1} + \Mat{D}\Mat{A}^{-1}\Mat{B})^{-1}\Mat{D}\Mat{A}^{-1},
  $}
	\end{equation}
we expand $\Mat{K}$ by substituting $\Mat{A} = \Mat{I}$, $\Mat{B} = \Mat{X}_{v}$, $\Mat{E} = \Mat{X}_{u}^{\pinv}\Mat{X}_{u}^{\pinv\T}$, and $\Mat{D} = \Mat{X}_{v}^{\T}$, leading to
\begin{equation}\label{eq_app_K2}
	\Mat{K} = \Mat{I} - \Mat{X}_{v}(\Mat{X}_{u}^{\T}\Mat{X}_{u} + \Mat{X}_{v}^{\T}\Mat{X}_{v})^{-1}\Mat{X}_{v}^{\T}.
\end{equation}
Similarly,
\begin{equation}\label{eq_app_tK2}
	\Mat{\tilde{K}} = \Mat{I} - \Mat{X}_{u}(\Mat{X}_{u}^{\T}\Mat{X}_{u} + \Mat{X}_{v}^{\T}\Mat{X}_{v})^{-1}\Mat{X}_{u}^{\T}.
\end{equation}
Thus,
\begin{equation}
	\Mat{X}^{\pinv} = \begin{bmatrix}
		\Mat{X}_{u}^{\pinv\T} - \Mat{X}_{u}^{\pinv\T}\Mat{X}_{v}^{\T}\Mat{K}\Mat{X}_{v}\Mat{X}_{u}^{\pinv}\Mat{X}_{u}^{\pinv\T} \\
		\Mat{X}_{v}^{\pinv\T} - \Mat{X}_{v}^{\pinv\T}\Mat{X}_{u}^{\T}\Mat{\tilde{K}}\Mat{X}_{u}\Mat{X}_{v}^{\pinv}\Mat{X}_{v}^{\pinv\T}
	\end{bmatrix}^{\T}.
\end{equation}
Let $\Mat{X}^{\pinv} = \begin{bmatrix} \Mat{\bar{U}} & \Mat{\bar{V}} \end{bmatrix}$, we have
\begin{align} \label{eq_app_UV2} \nonumber
		\Mat{\bar{U}} &= \left(\Mat{X}_{u}^{\pinv\T} - \Mat{X}_{u}^{\pinv\T}\Mat{X}_{v}^{\T}\Mat{K}\Mat{X}_{v}\Mat{X}_{u}^{\pinv}\Mat{X}_{u}^{\pinv\T}\right)^{\T} 
		\\ \nonumber  &= \Mat{X}_{u}^{\pinv} - \Mat{X}_{u}^{\pinv}\Mat{X}_{u}^{\pinv\T}\Mat{X}_{v}^{\T}\Mat{K}^{\T}\Mat{X}_{v}\Mat{X}_{u}^{\pinv} \\ \nonumber
		\Mat{\bar{V}} &= \left(\Mat{X}_{v}^{\pinv\T} - \Mat{X}_{v}^{\pinv\T}\Mat{X}_{u}^{\T}\Mat{\tilde{K}}\Mat{X}_{u}\Mat{X}_{v}^{\pinv}\Mat{X}_{v}^{\pinv\T}\right)^{\T}
		\\ &= \Mat{X}_{v}^{\pinv} - \Mat{X}_{v}^{\pinv}\Mat{X}_{v}^{\pinv\T}\Mat{X}_{u}^{\T}\Mat{\tilde{K}}^{\T}\Mat{X}_{u}\Mat{X}_{v}^{\pinv}
\end{align}
Substitute $\Mat{K}$ and $\Mat{\tilde{K}}$ with \eqref{eq_app_K2} and \eqref{eq_app_tK2}, we may rewrite \eqref{eq_app_UV2} into
\begin{equation*}\nonumber
\resizebox{0.45 \textwidth}{!}{$
		\Mat{\bar{U}} = \Mat{X}_{u}^{\pinv} - \Mat{X}_{u}^{\pinv}\Mat{X}_{u}^{\pinv\T}\Mat{X}_{v}^{\T}\left(\Mat{I} + \Mat{X}_{v}(\Mat{X}_{u}^{\T}\Mat{X}_{u} + \Mat{X}_{v}^{\T}\Mat{X}_{v})^{-1}\Mat{X}_{v}^{\T}\right)^{\T}\Mat{X}_{v}\Mat{X}_{u}
  $},
\end{equation*}
\begin{equation}
\resizebox{0.45 \textwidth}{!}{$
		\Mat{\bar{V}} = \Mat{X}_{v}^{\pinv} - \Mat{X}_{v}^{\pinv}\Mat{X}_{v}^{\pinv\T}\Mat{X}_{u}^{\T}\left(\Mat{I} + \Mat{X}_{u}(\Mat{X}_{u}^{\T}\Mat{X}_{u} + \Mat{X}_{v}^{\T}\Mat{X}_{v})^{-1}\Mat{X}_{u}^{\T}\right)^{\T}\Mat{X}_{u}\Mat{X}_{v}^{\pinv}
$}.
\end{equation}
That is, 
\begin{align} \nonumber
		\Mat{\bar{U}} = &\Mat{X}_{u}^{\pinv} - \Mat{X}_{u}^{\pinv}\Mat{X}_{u}^{\pinv\T}\Mat{X}_{v}^{\T}\Mat{X}_{v}\Mat{X}_{u}^{\pinv}\\ \nonumber &+ \Mat{X}_{u}^{\pinv}\Mat{X}_{u}^{\pinv\T}\Mat{X}_{v}^{\T}\Mat{X}_{v}(\Mat{X}_{u}^{\T}\Mat{X}_{u} + \Mat{X}_{v}^{\T}\Mat{X}_{v})^{-1}\Mat{X}_{v}^{\T}\Mat{X}_{v}\Mat{X}_{u}^{\pinv} \\ \nonumber
		\Mat{\bar{V}} = &\Mat{X}_{v}^{\pinv} - \Mat{X}_{v}^{\pinv}\Mat{X}_{v}^{\pinv\T}\Mat{X}_{u}^{\T}\Mat{X}_{u}\Mat{X}_{v}^{\pinv} \\ &+ \Mat{X}_{v}^{\pinv}\Mat{X}_{v}^{\pinv\T}\Mat{X}_{u}^{\T}\Mat{X}_{u}(\Mat{X}_{u}^{\T}\Mat{X}_{u} + \Mat{X}_{v}^{\T}\Mat{X}_{v})^{-1}\Mat{X}_{u}^{\T}\Mat{X}_{u}\Mat{X}_{v}^{\pinv}.
\end{align}
Let
\begin{equation}
	\begin{cases}
			\Mat{C}_{u} = \Mat{X}_{u}^{\T}\Mat{X}_{u} \\
			\Mat{C}_{v} = \Mat{X}_{v}^{\T}\Mat{X}_{v}
		\end{cases}, \text{and}
            \begin{cases}
			\Mat{R}_{u} = \Mat{R}_{u}^{-1} \\
			\Mat{R}_{v} = \Mat{R}_{v}^{-1}
		\end{cases}.
\end{equation}
we have
\begin{equation}
	\begin{cases}
		\Mat{\bar{U}} = \left[\Mat{I} - \Mat{R}_{u}\Mat{C}_{v} + \Mat{R}_{u}\Mat{C}_{v}(\Mat{C}_{u} + \Mat{C}_{v})^{-1}\Mat{C}_{v}\right]\Mat{X}_{u}^{\pinv}  \\
		\Mat{\bar{V}} = \left[\Mat{I} - \Mat{R}_{v}\Mat{C}_{u} + \Mat{R}_{v}\Mat{C}_{u}(\Mat{C}_{u} + \Mat{C}_{v})^{-1}\Mat{C}_{u}\right]\Mat{X}_{v}^{\pinv}
	\end{cases}.
\end{equation}
Thus,
\begin{equation}
	\Mat{X}^{\pinv} = \begin{bmatrix}
		\Mat{X}_{u} \\
		\Mat{X}_{v}
	\end{bmatrix}^{\pinv} = \begin{bmatrix} \Mat{\bar{U}} & \Mat{\bar{V}} \end{bmatrix},
\end{equation}
which completes the proof.

\end{proof}
\section{Proof of Theorem 1}\label{app_thm1}
\begin{proof}
	As indicated in Lemma \ref{corollary_decomp}, we have 
	\begin{equation}
		\Mat{X}^{\pinv} = \begin{bmatrix} \Mat{\bar{U}} & \Mat{\bar{V}} \end{bmatrix}
	\end{equation}
	where
	\begin{equation}\label{eq_app_UV3}
		\begin{cases}
			\Mat{\bar{U}} = \left[\Mat{I} - \Mat{R}_{u}\Mat{C}_{v} + \Mat{R}_{u}\Mat{C}_{v}(\Mat{C}_{u} + \Mat{C}_{v})^{-1}\Mat{C}_{v}\right]\Mat{X}_{u}^{\pinv}  \\
			\Mat{\bar{V}} = \left[\Mat{I} - \Mat{R}_{v}\Mat{C}_{u} + \Mat{R}_{v}\Mat{C}_{u}(\Mat{C}_{u} + \Mat{C}_{v})^{-1}\Mat{C}_{u}\right]\Mat{X}_{v}^{\pinv}
		\end{cases},
	\end{equation}
	and
	\begin{equation*} \nonumber
		\begin{cases}
			\Mat{R}_{u} = (\Mat{X}_{u}^{\T}\Mat{X}_{u})^{-1} = \Mat{X}_{u}^{\pinv}\Mat{X}_{u}^{\pinv\T} \\
			\Mat{R}_{v} = (\Mat{X}_{v}^{\T}\Mat{X}_{v})^{-1} = \Mat{X}_{v}^{\pinv}\Mat{X}_{v}^{\pinv\T}
		\end{cases} 
  \end{equation*}
  \begin{equation}
		\begin{cases}
			\Mat{C}_{u} = \Mat{R}_{u}^{-1} = \Mat{X}_{u}^{\T}\Mat{X}_{u} \\
			\Mat{C}_{v} = \Mat{R}_{v}^{-1} = \Mat{X}_{v}^{\T}\Mat{X}_{v}
		\end{cases}.
	\end{equation}
	Hence, 
	\begin{align} \nonumber
		\Mat{W}&=\Mat{X}^{\pinv}\Mat{Y} = \begin{bmatrix} \Mat{\bar{U}} & \Mat{\bar{V}} \end{bmatrix} \begin{bmatrix}
			\Mat{Y}_{u} \\
			\Mat{Y}_{v}
		\end{bmatrix}\\\label{eq_app_UY1}
		&=\Mat{\bar{U}}\Mat{Y}_{u} + \Mat{\bar{V}}\Mat{Y}_{v}. 
	\end{align}
	By substituting $\Mat{\bar{U}}$ and $\Mat{\bar{V}}$ with those in \eqref{eq_app_UV3}, we rewrite \eqref{eq_app_UY1} into
		\begin{align}\label{eq_app_XY1}\nonumber
			\Mat{\hat{W}}
			=   &\left[\Mat{I}  - \Mat{R}_{u}\Mat{C}_{v} + \Mat{R}_{u}\Mat{C}_{v}(\Mat{C}_{u} + \Mat{C}_{v})^{-1}\Mat{C}_{v}\right]\Mat{X}_{u}^{\pinv}\Mat{Y}_{u}\\
			& + \left[\Mat{I} - \Mat{R}_{v}\Mat{C}_{u} + \Mat{R}_{v}\Mat{C}_{u}(\Mat{C}_{u} + \Mat{C}_{v})^{-1}\Mat{C}_{u}\right]\Mat{X}_{v}^{\pinv}\Mat{Y}_{v}.
		\end{align}
As $\Mat{\hat{W}}_{u} = \Mat{X}_{u}^{\pinv}\Mat{Y}_{u}$ and $\Mat{\hat{W}}_{v} = \Mat{X}_{v}^{\pinv}\Mat{Y}_{v}$, \eqref{eq_app_XY1} can be rewritten as
\begin{align}\nonumber
	\Mat{\hat{W}}	= & \left[\Mat{I} - \Mat{R}_{u}\Mat{C}_{v} + \Mat{R}_{u}\Mat{C}_{v}(\Mat{C}_{u} + \Mat{C}_{v})^{-1}\Mat{C}_{v}\right]\Mat{\hat{W}}_{u} \\
	&+ \left[\Mat{I} - \Mat{R}_{v}\Mat{C}_{u} + \Mat{R}_{v}\Mat{C}_{u}(\Mat{C}_{u} + \Mat{C}_{v})^{-1}\Mat{C}_{u}\right]\Mat{\hat{W}}_{v}.
\end{align}
That is,
\begin{align}
	\Mat{\hat{W}}= \Mat{\mathcal{W}}_{u}\Mat{\hat{W}}_{u} + \Mat{\mathcal{W}}_{v}\Mat{\hat{W}}_{v}&,
\end{align}
where
\begin{equation*}\nonumber
	\begin{cases}
		\Mat{\mathcal{W}}_{u} = \Mat{I} - \Mat{R}_{u}\Mat{C}_{v} + \Mat{R}_{u}\Mat{C}_{v}(\Mat{C}_{u} + \Mat{C}_{v})^{-1}\Mat{C}_{v}\\
		\Mat{\mathcal{W}}_{v} = \Mat{I} - \Mat{R}_{v}\Mat{C}_{u} + \Mat{R}_{v}\Mat{C}_{u}(\Mat{C}_{u} + \Mat{C}_{v})^{-1}\Mat{C}_{u}
	\end{cases},
 \end{equation*}
 \begin{equation}
	\begin{cases}
            \Mat{C}_{u} = \Mat{X}_{u}^{\T}\Mat{X}_{u}\\
            \Mat{C}_{v} = \Mat{X}_{v}^{\T}\Mat{X}_{v}
        \end{cases}
	\text{and} \quad
	\begin{cases}
            \Mat{R}_{u} = \Mat{C}_{u}^{-1}\\
            \Mat{R}_{v} = \Mat{C}_{v}^{-1}
        \end{cases}.
\end{equation}
\end{proof}

\section{Proof of Theorem 2}\label{app_thm2}
\begin{proof}
    First we consider the aggregation of two clients. Directly substituting $\hat{\Mat{W}}_{u}$ as 
 $\hat{\Mat{W}}_{u}^{\textup{r}}$ and changing $\Mat{C}_{u}$, $\Mat{C}_{v}$ to $\Mat{C}_{u}^{\textup{r}} = (\Mat{X}_{u}^{\T}\Mat{X}_{u} + \gamma \Mat{I})$, $\Mat{C}_{v}^{\textup{r}} = (\Mat{X}_{v}^{\T}\Mat{X}_{v} + \gamma \Mat{I})$ in Theorem \ref{app_thm1}, we have 
\begin{align} \label{eq_thm1_r}
	\Mat{\hat{W}}^{\textup{r}}= \Mat{\mathcal{W}}_{u}^{\textup{r}}\Mat{\hat{W}}_{u}^{\textup{r}} + \Mat{\mathcal{W}}_{v}^{\textup{r}}\Mat{\hat{W}}_{v}^{\textup{r}},
\end{align}
where 
 \begin{equation*}\nonumber
	\begin{cases}
		\Mat{\mathcal{W}}_{u}^{\textup{r}} = \Mat{I} - \Mat{R}_{u}^{\textup{r}}\Mat{C}_{v}^{\textup{r}} + \Mat{R}_{u}^{\textup{r}}\Mat{C}_{v}^{\textup{r}}(\Mat{C}_{u}^{\textup{r}} + \Mat{C}_{v}^{\textup{r}})^{-1}\Mat{C}_{v}^{\textup{r}}\\
		\Mat{\mathcal{W}}_{v}^{\textup{r}} = \Mat{I} - \Mat{R}_{v}^{\textup{r}}\Mat{C}_{u}^{\textup{r}} + \Mat{R}_{v}^{\textup{r}}\Mat{C}_{u}^{\textup{r}}(\Mat{C}_{u}^{\textup{r}} + \Mat{C}_{v}^{\textup{r}})^{-1}\Mat{C}_{u}^{\textup{r}}
	\end{cases},
 \end{equation*}
  \begin{equation}
	\begin{cases}
		\Mat{C}_{u}^{\textup{r}} = (\Mat{X}_{u}^{\T}\Mat{X}_{u} + \gamma \Mat{I})\\
		\Mat{C}_{v}^{\textup{r}} = (\Mat{X}_{v}^{\T}\Mat{X}_{v} + \gamma \Mat{I})
	\end{cases}\text{and} 
	\begin{cases}
		\Mat{R}_{u}^{\textup{r}} = \Mat{C}_{u}^{\textup{r} -1}\\
		\Mat{R}_{v}^{\textup{r}} = \Mat{C}_{v}^{\textup{r} -1}
	\end{cases}.
 \end{equation}
Since $\Mat{\hat{W}}_{u}^{\textup{r}} = (\Mat{X}_{u}^{\T}\Mat{X}_{u} + \gamma \Mat{I})^{-1}\Mat{X}_{u}^{\T}\Mat{Y}_{u}$, then 
\begin{equation*}\nonumber
\resizebox{0.43\textwidth}{!}{$
     \Mat{\mathcal{W}}_{u}^{\textup{r}}\Mat{\hat{W}}_{u}^{\textup{r}}= [\Mat{I} - \Mat{R}_{u}^{\textup{r}}\Mat{C}_{v}^{\textup{r}} + \Mat{R}_{u}^{\textup{r}}\Mat{C}_{v}^{\textup{r}}(\Mat{C}_{u}^{\textup{r}} + \Mat{C}_{v}^{\textup{r}})^{-1}\Mat{C}_{v}^{\textup{r}}] \Mat{R}_{u}^{\textup{r}} \Mat{X}_{u}^{\T}\Mat{Y}_{u} $}
\end{equation*}
\begin{equation}\label{eq_agg_u_1}
\resizebox{0.43\textwidth}{!}{$
     = [\Mat{R}_{u}^{\textup{r}} - \Mat{R}_{u}^{\textup{r}}\Mat{C}_{v}^{\textup{r}}\Mat{R}_{u}^{\textup{r}} + \Mat{R}_{u}^{\textup{r}}\Mat{C}_{v}^{\textup{r}}(\Mat{C}_{u}^{\textup{r}} + \Mat{C}_{v}^{\textup{r}})^{-1}\Mat{C}_{v}^{\textup{r}}\Mat{R}_{u}^{\textup{r}}]  \Mat{X}_{u}^{\T}\Mat{Y}_{u}.$}
\end{equation}

According to the Woodbury Matrix Identity in \eqref{eq_woodbury}, let $\Mat{B} = \Mat{I}, \Mat{D} = \Mat{I}$, we have 
\begin{equation}
    (\Mat{A} + \Mat{E})^{-1} = \Mat{A}^{-1} - \Mat{A}^{-1}(\Mat{A}^{-1}+\Mat{E}^{-1})^{-1}\Mat{A}^{-1}.
\end{equation}
 Then we have
\begin{equation}
    \Mat{A}^{-1}(\Mat{A}^{-1}+\Mat{E}^{-1})^{-1}\Mat{A}^{-1} = \Mat{A}^{-1} - (\Mat{A} + \Mat{E})^{-1} .
\end{equation}
Swapping $\Mat{A}^{-1}$ with $\Mat{C}_{v}^{\textup{r}}$ and $\Mat{E}^{-1}$ with $\Mat{C}_{u}^{\textup{r}}$, we have 
  \begin{equation}\label{eq_sub_C}
    \Mat{C}_{v}^{\textup{r}}(\Mat{C}_{v}^{\textup{r}}+\Mat{C}_{u}^{\textup{r}})^{-1}\Mat{C}_{v}^{\textup{r}} = \Mat{C}_{v}^{\textup{r}} - (\Mat{R}_{u}^{\textup{r}} + \Mat{R}_{v}^{\textup{r}})^{-1} .
\end{equation}
Similarly, 
\begin{equation}\label{eq_sub_R}
    \Mat{R}_{u}^{\textup{r}}(\Mat{R}_{u}^{\textup{r}}+\Mat{R}_{v}^{\textup{r}})^{-1}\Mat{R}_{r}^{\textup{r}} = \Mat{R}_{u}^{\textup{r}} - (\Mat{C}_{u}^{\textup{r}} + \Mat{C}_{v}^{\textup{r}})^{-1} .
\end{equation}
By substituting \eqref{eq_sub_C} into \eqref{eq_agg_u_1}, 
\begin{equation} \label{eq_agg_u_2}
\resizebox{0.48 \textwidth}{!}{$
\Mat{\mathcal{W}}_{u}^{\textup{r}}\Mat{\hat{W}}_{u}^{\textup{r}} = [\Mat{R}_{u}^{\textup{r}} - \Mat{R}_{u}^{\textup{r}}\Mat{C}_{v}^{\textup{r}}\Mat{R}_{u}^{\textup{r}} + \Mat{R}_{u}^{\textup{r}}\Mat{C}_{v}^{\textup{r}}\Mat{R}_{u}^{\textup{r}} - \Mat{R}_{u}^{\textup{r}}(\Mat{R}_{u}^{\textup{r}} + \Mat{R}_{v}^{\textup{r}})^{-1}\Mat{R}_{u}^{\textup{r}}] \Mat{X}_{u}^{\T}\Mat{Y}_{u}.
$}
\end{equation}
Then 
\begin{equation}
 \Mat{\mathcal{W}}_{u}^{\textup{r}}\Mat{\hat{W}}_{u}^{\textup{r}}= [\Mat{R}_{u}^{\textup{r}} - \Mat{R}_{u}^{\textup{r}}(\Mat{R}_{u}^{\textup{r}} + \Mat{R}_{v}^{\textup{r}})^{-1}\Mat{R}_{u}^{\textup{r}}] \Mat{X}_{u}^{\T}\Mat{Y}_{u} .
\label{eq_agg_u_3}
\end{equation}
Further substituting \eqref{eq_sub_R} into \eqref{eq_agg_u_3}, we have 
\begin{align}\nonumber
\Mat{\mathcal{W}}_{u}^{\textup{r}}\Mat{\hat{W}}_{u}^{\textup{r}}  &= [\Mat{R}_{u}^{\textup{r}} - \Mat{R}_{u}^{\textup{r}}(\Mat{R}_{u}^{\textup{r}} + \Mat{R}_{v}^{\textup{r}})^{-1}\Mat{R}_{u}^{\textup{r}}] \Mat{X}_{u}^{\T}\Mat{Y}_{u} 
\\\nonumber &= [\Mat{R}_{u}^{\textup{r}} - \Mat{R}_{u}^{\textup{r}} + (\Mat{C}_{u}^{\textup{r}} + \Mat{C}_{v}^{\textup{r}})^{-1}] \Mat{X}_{u}^{\T}\Mat{Y}_{u} 
\\\nonumber &= (\Mat{C}_{u}^{\textup{r}} + \Mat{C}_{v}^{\textup{r}})^{-1} \Mat{X}_{u}^{\T}\Mat{Y}_{u}
\\ &= (\Mat{C}_{u} + \Mat{C}_{v} + 2 \gamma \Mat{I})^{-1} \Mat{X}_{u}^{\T}\Mat{Y}_{u}.
\label{eq_agg_u_4}
\end{align}
Similarly, 
\begin{equation}
    \Mat{\mathcal{W}}_{v}^{\textup{r}}\Mat{\hat{W}}_{v}^{\textup{r}} = (\Mat{C}_{u} + \Mat{C}_{v} + 2 \gamma \Mat{I})^{-1} \Mat{X}_{v}^{\T}\Mat{Y}_{v}.
\end{equation}
Thus equation \eqref{eq_thm1_r} can be converted to 
\begin{align} 
	\Mat{\hat{W}}^{\textup{r}}= (\Mat{C}_{u} + \Mat{C}_{v} + 2 \gamma \Mat{I})^{-1} (\Mat{X}_{u}^{\T}\Mat{Y}_{u}+\Mat{X}_{v}^{\T}\Mat{Y}_{v}).
\end{align}
Since $\Mat{\hat{W}} = \Mat{X}^{\pinv} \Mat{Y}$ and $\Mat{X} = \begin{bmatrix} \Mat{X}_{u} \\ \Mat{X}_{v} \end{bmatrix}$, $\Mat{Y} = \begin{bmatrix} \Mat{Y}_{u} \\ \Mat{Y}_{v} \end{bmatrix}$, with full-column rank of $\Mat{X}$, 
\begin{align} \label{eq_thm1_C}
    \Mat{\hat{W}} &= \Mat{X}^{\pinv}\Mat{Y} = (\Mat{X}^{\textup{T}}\Mat{X})^{-1}\Mat{X}^{\textup{T}}Y 
\\\nonumber &= (\begin{bmatrix} \Mat{X}_{u}^{\textup{T}} \quad \Mat{X}_{v}^{\textup{T}} \end{bmatrix}\begin{bmatrix} \Mat{X}_{u} \\ \Mat{X}_{v} \end{bmatrix})^{-1}\begin{bmatrix} \Mat{X}_{u}^{\textup{T}} \quad \Mat{X}_{v}^{\textup{T}} \end{bmatrix} \begin{bmatrix} \Mat{Y}_{u} \\ \Mat{Y}_{v} \end{bmatrix})
\\\nonumber &= (\Mat{X}_{u}^{\T}\Mat{X}_{u}+\Mat{X}_{v}^{\T}\Mat{X}_{v})^{-1}(\Mat{X}_{u}^{\T}\Mat{Y}_{u}+\Mat{X}_{v}^{\T}\Mat{Y}_{v})
\\\nonumber &= (\Mat{C}_{u}+\Mat{C}_{v})^{-1} (\Mat{X}_{u}^{\T}\Mat{Y}_{u}+\Mat{X}_{v}^{\T}\Mat{Y}_{v}).
\end{align}
By comparing with \eqref{eq_thm1_r} and \eqref{eq_thm1_C}, we can obtain the relation between $\Mat{\hat{W}}$ and $\Mat{\hat{W}}^{\textup{r}}$ as follows.

\begin{equation}
    \Mat{\hat{W}}^{\textup{r}} = (\Mat{C}_{u}^{\textup r} + \Mat{C}_{v}^{\textup r})^{-1} (\Mat{C}_{u} + \Mat{C}_{v}) \Mat{\hat{W}}.
\end{equation}

By extending to the multi-client scenario, we have 
\begin{equation}
\hat{\Mat{W}}_{\textup{agg},k}^{\text{r}} = (\Mat{C}_{\textup{agg},k}^{\textup{r}})^{-1} \Mat{C}_{\textup{agg},k} \hat{\Mat{W}}_{\textup{agg},k},
	\end{equation}
	where
	\begin{equation*}
\Mat{C}_{\textup{agg},k}^{\textup{r}} = \Mat{C}_{\textup{agg},k} + k\gamma \Mat{I} = \sum_{i}^{k} \Mat{C}_{i}^{\textup{r}},
	\end{equation*}
 \begin{equation}
 \Mat{C}_{i}^{\textup{r}} = \Mat{X}_{i}^{\T}\Mat{X}_{i} +\gamma \Mat{I},
 \end{equation}
 which complete the proof.
\end{proof}

% \section{Details of Validating Dummy Data}

%\input{App_Dummy}
\begin{table*}[t!]
	\caption{Deviation $\Delta\Mat{W}$ between the joint-trained weight and the aggregated one (average of 5 runs).
 }
 % \resizebox{1\textwidth}{!}{
	\begin{tabular}{lcccccc}
		\toprule
		Difference & $K=2$& $K=10$ & $K=20$&$K=50$&$K=100$ & $K=200$\\ \hline   
		$\Delta \Mat{W}$ (w/o RI) & $7.83\times10^{-14}$& $1.76\times10^{-12}$ & $9.86\times10^{-1}$& $5.90$& $5.93\times10^{4}$ & $3.67\times10^{12}$\\ \hline      
		% $\lVert \Mat{\hat{W}} - \Mat{\hat{W}_\text{agg}^{\text{rg}}} \rVert_{1}$ &$2.98\times10^{-3}$& $2.97\times10^{-2}$& $7.40\times10^{-2}$& $1.47\times10^{-1}$ \\ \hline
		$\Delta \Mat{W}$ (w/ RI) &$4.94\times10^{-14}$& $1.74\times10^{-12}$& $5.09\times10^{-10}$& $8.45\times10^{-10}$& $7.57\times10^{-10}$ &  $7.81\times10^{-10}$ \\ \bottomrule
	\end{tabular}
 % }
 \centering  
	\label{table:dummy}
\end{table*}

\section{Validating AA Laws on Dummy Dataset} \label{Dummytest}
Here we validate the AA laws in AFL, whether the aggregated weight $\Mat{\hat{W}}_{\text{agg},K}$ equals to $\Mat{\hat{W}}$ trained on a centralized dataset. This is done by measuring the deviation (i.e., $\Delta\Mat{W} = \lVert \Mat{\hat{W}} - \Mat{\hat{W}}_{\text{agg},K} \rVert_{1}$) between the joint-trained weight and the aggregated one on a dummy dataset.

\textbf{Dummy Dataset.} We randomly generate a 512-dimension and 10,000-sample dummy dataset. This dataset has 10 classes, with each class containing an identical number of samples. The samples in the dummy dataset are randomly but evenly distributed to $K$ clients (we set $K=2, 10, 20, 50 , 100, 200$).

\textbf{Results.} As indicated in Table \ref{table:dummy}, without the RI process, the deviation is negligible for $K=2,10$, but it grows with an increasing $K$ and could become unacceptable (e.g., $3.67\times10^{12}$ for $K=200$). This is because the full-column rank assumption might not hold anymore for large $K$. By adopting the RI process, the deviations become negligible (around $10^{-10}$) for various $K$ values as shown in the second row of Table \ref{table:dummy}. The RI process introduces $\gamma$ (we adopt $\gamma=1$ in this case, but any value would suffice) to satisfy full-column rank condition, which is later removed in \eqref{eq_rf} to restore the AA law's optimality. This experiment has well demonstrated AFL's invariance to data partitioning with empirical evidence. The codes for the dummy data validation can be found in the file $App\_Dummy.ipynb$ in the released code.

\section{Implementation Details of Experiments}\label{app_compared}
For the compared methods, we set the local epoch to 1 and all the clients are selected to participate each round after local training. The batch size is set to 64 and we employ SGD optimizer with learning rate of 0.05. The number of global communication rounds is set to be 500 since there is little or no performance gain with more rounds. We report the average and standard deviation of best top-1 accuracy in three runs. All the experiments are conducted on a NVIDIA GeForce RTX 4090 GPU. All the compared methods except FedDisco are implemented with PFLlib \cite{PFLlib2023arxiv} and FedDisco is implemented upon FedAvg via official codes.

For the specified hyperparameters in compared methods, we tune the parameters via grid search. For FedProx \cite{FedProx2020MLSYS}, we tune the hyperparameter $\mu$ from $\{0.0001, 0.001, 0.01, 0.1\}$. For MOON \cite{MOON2021CVPR}, we tune the hyperparameter $\mu$ from $\{0.1, 1, 5, 10\}$. For FedDyn \cite{FedDyn2021ICLR}, we tune the hyperparameter $\alpha$ from $\{0.001, 0.01, 0.1, 1.0\}$. For FedNTD \cite{FedNTD2022NeurIPS}, we tune the hyperparameters $\tau$ from $\{0.1, 0.5, 1.0\}$ and $\beta$ from $\{0.1, 0.5, 1.0, 2.0\}$. For FedDisco \cite{FedDisco2023ICML}, we tune the hyperparameters $a$ from $\{0.01, 0.05, 0.1, 0.5\}$ and $b$ from $\{0.005, 0.01, 0.05, 0.1\}$. The best parameters we adopted are, $\mu=0.001$ in FedProx \cite{FedProx2020MLSYS}, $\mu=1$ in MOON \cite{MOON2021CVPR},   $\alpha = 1.0$ in FedDyn \cite{FedDyn2021ICLR}, $\tau = 0.5, \beta=1.0$ in FedNTD \cite{FedNTD2022NeurIPS} and $a = 0.05, b = 0.01$ in FedDisco \cite{FedDisco2023ICML}. 

\section{Necessity of FL with Pre-trained Model}
To validate the necessity of FL with pre-trained backbone, we train the models locally without aggregation under the setting of $\alpha$=0.1, K=100. We report the average and maximum test accuracy of local training among all the clients. As shown in Table. \ref{table:r3}, without aggregation, the results of local training (12.04\% and 16.36\%) fall behind FedAvg (56.57\%) and AFL (58.56\%) with large gaps. Training local models without FL could suffer from the data heterogeneity and the collaboration among clients is beneficial. This pattern is also validated in previous studies with pre-trained backbone \cite{FedPR2023CVPR}. 

\begin{table}[!h]
% \footnotesize
	   % \vspace{-10pt}
	\captionof{table}{Comparison between FL teniques including FedAvg and AFL with local training when utilizing pre-trained backbone.}
    % \captionof{table}{The top-1 classification accuracy (\%) of AFL and FedAvg under different data heterogeneity.}
	% \resizebox{1\textwidth}{!}{
     % \vspace{-10pt}
    \begin{tabular}{lcccc} \toprule
		Methods& Local Max & Local Avg & FedAvg & AFL  \\ \hline
		Acc.(\%)  & 16.36 & 12.04 & 56.57& \textbf{58.56} \\ \bottomrule
	\end{tabular}
     % \vspace{-10pt}
 % }
 \centering
	\label{table:r3}
\end{table}

\section{Comparative Study with Single-Round FL}
In this section, we provide a comparative study with another single-round FL technique FedFisher \cite{FedFisher2024AISTAT} to further validate our proposed AFL. We compare AFL with FedFisher under the setting of $\alpha=0.1, K=50$ (larger K will lead to out-of-memory in FedFisher) with the same pre-trained ResNet-18 provided by the repository of FedFisher. As shown in Table \ref{table:r1}, AFL outperforms FedFisher with considerable distance (35.87\% v.s. 19.31\%). FedFisher utilizes iterative gradient-descent to aggregate local weights and MSE loss is established by the difference of global and local weights to preventing drifting between them during aggregation. However, this technique could still suffer from the data heterogeneity, while AFL formulates the AA law to achieve the invariance to data partitioning, enabling outperforming result when compared with FedFisher.

\begin{table}[!h]
% \footnotesize
	  % \vspace{-10pt}
	\captionof{table}{Comparative study between AFL and  FedFisher.}
    % \captionof{table}{The top-1 classification accuracy (\%) of AFL and FedAvg under different data heterogeneity.}
	% \resizebox{1\textwidth}{!}{
    % \vspace{-10pt}
    \begin{tabular}{lcc} \toprule
		Methods& FedFisher & AFL  \\ \hline
		Acc.(\%)  & 19.31 & 35.87 \\ \bottomrule
	\end{tabular}
 % }
 \centering
	\label{table:r1}
    % \vspace{-10pt}
\end{table}

% {
%     \small
%     \bibliographystyle{ieeenat_fullname}
%     \bibliography{afl}
% }

% \end{appendices}

% \section{Rationale}
% \label{sec:rationale}
% % 
% Having the supplementary compiled together with the main paper means that:
% % 
% \begin{itemize}
% \item The supplementary can back-reference sections of the main paper, for example, we can refer to \cref{sec:intro};
% \item The main paper can forward reference sub-sections within the supplementary explicitly (e.g. referring to a particular experiment); 
% \item When submitted to arXiv, the supplementary will already included at the end of the paper.
% \end{itemize}
% % 
% To split the supplementary pages from the main paper, you can use \href{https://support.apple.com/en-ca/guide/preview/prvw11793/mac#:~:text=Delete%20a%20page%20from%20a,or%20choose%20Edit%20%3E%20Delete).}{Preview (on macOS)}, \href{https://www.adobe.com/acrobat/how-to/delete-pages-from-pdf.html#:~:text=Choose%20%E2%80%9CTools%E2%80%9D%20%3E%20%E2%80%9COrganize,or%20pages%20from%20the%20file.}{Adobe Acrobat} (on all OSs), as well as \href{https://superuser.com/questions/517986/is-it-possible-to-delete-some-pages-of-a-pdf-document}{command line tools}.

\end{document}